\documentclass[fleqn,preprint,3p,a4paper]{elsarticle}

\usepackage{pifont}
\usepackage{amssymb,times,amsmath,xspace,dsfont,pgf,latexsym,natbib}
\usepackage[latin1]{inputenc}
\usepackage[english]{babel}
\journal{Fuzzy Sets and Systems}

\usepackage[normalem]{ulem}
\usepackage[all]{xy}

\usepackage[latin1]{inputenc}

\newtheorem{definition}{Definition}[section]
\newtheorem{example}{Example}[section]
\newtheorem{remark}{Remark}[section]
\newtheorem{lemma}{Lemma}[section]
\newtheorem{corollary}{Corollary}[section]
\newtheorem{proposition}{Proposition}[section]
\newtheorem{theorem}{Theorem}[section]

\newproof{proof}{Proof}

\newcommand{\xmark}{\ding{55}}%

\begin{document}

\begin{frontmatter}

\title{Towards interval uncertainty propagation control in bivariate aggregation processes and the introduction of width-limited interval-valued overlap functions}

\author[upna,furg]{Tiago da Cruz Asmus\corref{cor}}

\cortext[cor]{Corresponding author. Phone: +34 610031770}
 \ead{tiago.dacruz@unavarra.es, tiagoasmus@furg.br}

\author[upna,furg2]{Graçaliz Pereira Dimuro}
 \ead{gracaliz@gmail.com, gracalizdimuro@furg.br, gracaliz.pereira@unavarra.es}

\author[ufrn,isc]{Benjamín Bedregal}
\ead{bedregal@dimap.ufrn.br}

\author[upna,isc]{José Antonio Sanz}
\ead{joseantonio.sanz@unavarra.es}

\author[slovak]{Radko Mesiar}
\ead{mesiar@math.sk}

\author[upna,isc]{Humberto Bustince}
\ead{bustince@unavarra.es}

\address[upna]{Departamento de Estadística, Informática y Matemáticas, Universidad Pública de Navarra\\ Campus Arrosadia s/n, 31006, Pamplona, Spain}

\address[furg]{Instituto de Matemática, Estatística e Física,  Universidade Federal do Rio Grande\\
	Av. Itália km 08,  Campus Carreiros,  96201-900 Rio Grande, Brazil}

\address[furg2]{Centro de Ciência Computacionais,  Universidade Federal do Rio Grande\\
Av. Itália km 08,  Campus Carreiros,  96201-900 Rio Grande, Brazil}

\address[ufrn]{Departamento de Informática e
	Matemática Aplicada,  Universidade Federal do Rio Grande do Norte\\ Campus Universitário s/n, 59072-970 Natal, Brazil}

\address[isc]{Institute of Smart Cities, Universidad Pública de Navarra\\ Campus Arrosadia s/n, 31006, Pamplona, Spain}

\address[slovak]{Slovak University of Technology,  Bratislava, Slovakia\\
Institute of Information Theory and Automation, Academy of Sciences of the Czech Republic, Prague, Czech Republic}

\begin{abstract}
Overlap functions are a class of aggregation functions that measure the overlapping degree between two values. They have been successfully applied as a fuzzy conjunction operation in several problems in which associativity is not required, such as image processing and classification. Interval-valued overlap functions were defined as an extension to express the overlapping of interval-valued data, and they have been usually applied when there is uncertainty regarding the assignment of membership degrees, as in interval-valued fuzzy rule-based classification systems. In this context, the choice of a total order for intervals can be significant, which motivated the recent developments on interval-valued aggregation functions and interval-valued overlap functions that are increasing to a given admissible order, that is, a total order that refines the usual partial order for intervals. Also, width preservation has been considered on these recent works, in an intent to avoid the uncertainty increase and guarantee the information quality, but no deeper study was made regarding the relation between the widths of the input intervals and the output interval, when applying interval-valued functions, or how one can control such uncertainty propagation based on this relation. Thus, in this paper we: (i) introduce and develop the concepts of width-limited interval-valued functions and width limiting functions, presenting a theoretical approach to analyze the relation between the widths of the input and output intervals of bivariate interval-valued functions, with special attention to interval-valued aggregation functions; (ii) introduce the concept of $(a,b)$-ultramodular aggregation functions, a less restrictive extension of one-dimension convexity for bivariate aggregation functions, which have an important predictable behaviour with respect to the width when extended to the interval-valued context; (iii) define width-limited interval-valued overlap functions, taking into account a function that controls the width of the output interval and a new notion of increasingness with respect to a pair of partial orders $(\leq_1,\leq_2)$; (iv) present and compare three construction methods for these width-limited interval-valued overlap functions, considering a pair of orders $(\leq_1,\leq_2)$, which may be admissible or not, showcasing the adaptability of our developments.
\end{abstract}


\end{frontmatter}

\section{Introduction}\label{intro}

Aggregation functions are useful operators that combine (fuse) several numerical values into a single representative one, being especially suitable to model fuzzy logic operations and they have been widely employed in several theoretical and applied fields \cite{grabisch2,Beliakov,BeliakovBustinceCalvo_book}.

Overlap functions are a particular class of aggregation functions that do not need to be associative, and they were originally defined as continuous functions in order to deal with the overlapping between classes in image processing problems \cite{Bus10a,Jurio201369,Bustince2021}. They have quickly risen in popularity due to some desirable properties that they present. In \cite{DIMURO201927,Dimuro2015}, one can find clear discussions on the advantages that overlap functions have over the popular t-norms. For example, overlap functions are closed to the convex sum and the aggregation by internal generalized composition, whereas t-norms are not. Also, overlap functions showed good results when applied in problems in which associativity of the employed aggregation operator is not required, as in fuzzy rule-based classification \cite{elkano, edurne, DIMURO202027,gcf1f2,8468107,Lucca2018-nc,LUCCA201894}, decision making \cite{ELKANO2017,GARCIAINDEX}, wavelet-fuzzy power quality diagnosis system \cite{NOLASCO2019284} or forest fire detection \cite{GARCIAJIMENEZ2017834}, among others.


In fuzzy modeling, there may be uncertainty regarding the values of the membership degrees or the definition of the membership functions to be employed in the system \cite{MENDEL2007988}. One possible solution is the adoption of interval-valued fuzzy sets (IVFSs) \cite{Bus7145399,Grattan-Guiness75,Zad75}, where the membership degrees are represented by intervals. In this manner, the width of the assigned intervals are intrinsically related with the uncertainty/ignorance with respect to the modeling of the fuzzy sets \cite{BEDREGAL20101373, Dimuro20113898,Sanz2}. IVFSs have been successfully applied in many different fields, such as classification \cite{Sanz3,Sanz2013,ASMUS2020}, image processing \cite{Bustince2009,GALARCOLOR}, game theory \cite{Asmus2016}, multicriteria decision making \cite{Barrenechea201433,KUTLUGUNDOGDU2019307}, pest
control \cite{leo-ifsa}, irrigation systems \cite{HU2020109950} and collaborative clustering \cite{NGO2018404}.

Interval-valued aggregation functions were defined in \cite{Marichal2002}, in order to model the aggregation of interval-valued data in the unit interval. Following a similar approach, interval-valued overlap functions were defined, independently, by Qiao and Hu \cite{Qiao201719} and Bedregal et al. \cite{BEDREGAL20171}, as an extension of overlap functions to the interval-valued context. By extending and generalizing the concept of interval-valued overlap functions, Asmus et al. \cite{ASMUS2020} introduced the concepts of n-dimensional interval-valued overlap functions and general interval-valued overlap functions.

It is noteworthy that most of the extensions of aggregation functions to the interval-valued context were made following the optimality (the least possible interval width) and correctness (the unknown value of the extended operation is contained in the resulting interval) criteria for interval representation (also called, the best interval representation), as discussed in \cite{BEDREGAL20101373, Dimuro2008123}, and taking into account the usual product order when comparing intervals \cite{moore2009}. Although this approach is both intuitive and theoretically sound, it may present some drawbacks on the application side:
\begin{description}
  \item[(i)] The product order is not a total order, meaning that one may face data that is not comparable, which is a serious hindrance in problems such as decision making and classification \cite{BUSTINCE201369};
   \item[(ii)] In order to ensure correctness, many interval-valued operations produce overestimation \cite{ML03}, which leads to increasing widths of the resulting intervals and, consequently, such intervals, although correct, usually carry no meaningful information \cite{DIM00}.
\end{description}

To solve the first problem \textbf{(i)}, avoiding a stalemate when comparing intervals, Bustince et al. \cite{BUSTINCE201369} introduced the concept of admissible orders, that is, total orders that refine the product order, and that can be constructed by a pair of aggregation functions. Since then, many works using admissible orders have appeared in the literature, for example, \cite{ZAPATA201791,BENTKOWSKA2015792}. In particular, Bustince et al. \cite{BUSTINCE202023} presented a construction method for interval-valued aggregation functions that are increasing with respect to a given admissible order. In a similar line of work, Asmus et al. \cite{ASMUS2020-TFS} introduced the concept of n-dimensional admissibly ordered interval-valued overlap functions, which are n-dimensional interval-valued overlap functions that are increasing with respect to an admissible order.

 In an initial attempt to deal with the second drawback \textbf{(ii)}, the construction method presented by Bustince et al. \cite{BUSTINCE202023} produces interval-valued aggregation functions that can be width-preserving whenever some restrictive conditions are satisfied. Note that the width of the interval output of a width-preserving interval-valued function is equal to the width of the interval inputs, when they all have the same width. However, Bustince et al. \cite{BUSTINCE202023} clearly state that, ideally, the definition of width preservation would have to take into account the width of the interval inputs in every case, not only when those inputs have the same length, which we call the drawback \textbf{(iii)} to be overcome.

Considering the second problem \textbf{(ii)}, Asmus et al. \cite{ASMUS2020-TFS} presented a construction method for n-dimensional admissibly ordered interval-valued overlap functions in which the width of the output is always less or equal to the minimal width of the inputs (see Theorem \ref{theo-benja} in Section \ref{sec-prelim}), which also comes to avoid the problem \textbf{(iii)}. However, this type of minimal width limitation has two sides: on one hand, as desired, the functions produced by the method avoid an increasing width (uncertainty) propagation; on the other hand, unfortunately, just one degenerate input interval (that is, with width equal to zero) is sufficient to completely remove all uncertainty of the output interval, which is clearly counterintuitive, to say the least.



Thus, the study of the relation between the width of the inputs and the output of interval-valued fuzzy operations coupled with adaptable tools to limit the increasing uncertainty in the output of such operations is still a challenge to overcome in the literature, especially regarding interval-valued aggregation and interval-valued overlap functions, which are of our particular interest.

The development of models  that help to avoid the overestimation\footnote{For a discussion on the problem of overestimation of interval-valued operations in practical applications, see \cite{doi:10.1080/15502287.2012.683134,doi:10.1080/15397734.2018.1548969}.} of the interval outputs' widths in practical applications certainly will increase the applicability of interval-valued fuzzy-based tools to solve many problems, as in fuzzy-rule based classification systems (see, e.g.: \cite{Sanz2,ASMUS2020-TFS,Sanz13}) and decision making (see, e.g.: \cite{CHEN2020,LAN2020146}), by providing interval outputs with better information quality. We point out that the information quality of interval-valued results is a strong requirement claimed by scientists and engineers interested in interval-based tools \cite{KK96}.

So, in order to present a contribution to solve the problems \textbf{(ii)} and \textbf{(iii)} in the context of interval-valued overlap functions, and even in a more general framework, without disregarding the problem \textbf{(i)}, this paper has the following general objective:

\begin{description}
  \item[$\bullet$] To develop a theoretical approach to aid the analysis of bivariate interval-valued operations with respect to the width of the operated intervals in order to control the uncertainty propagation, with special attention to interval-valued overlap functions, admissibly ordered or not.
\end{description}

To accomplish this goal, we have the following specific objectives:

\begin{description}
  \item[1)] To introduce the concepts of width-limited interval-valued functions and width-limiting functions, which are theoretical tools to study the relation between the widths of the inputs with the width of the output of interval-valued functions, necessary for the construction of interval-valued functions with controlled uncertainty propagation;


  \item[2)] To define $(a,b)$-ultramodular aggregation functions, a less restrictive extension of one-dimension convexity for bivariate aggregation functions, which shall have an important predictable behaviour with respect to their interval output widths when extended to the interval-valued context;

  \item[3)] To study the relation between some width-limited interval-valued functions and their respective width-limiting functions, especially when dealing with $(a,b)$-ultramodular aggregation functions, giving some backdrop for future comparisons with similarly constructed interval-valued functions;

  \item[4)] To define the notion of increasingness for a pair of partial orders, allowing for more flexible construction methods for width-limited interval-valued functions;

  \item[5)] To introduce the concept of width-limited interval-valued overlap functions, taking into account a width-limiting function and a pair of partial orders, which allows the definition of interval-valued overlap operations that provide output intervals that do not exceed a desirable uncertainty (width) threshold;

  \item[6)] To study the relation between width-limited interval-valued overlap functions and some of their width-limiting functions, particularly when considering the best interval representation of some overlap function;

  \item[7)] To present and study three construction methods for width-limited interval-valued overlap functions, presenting examples and comparisons between them  to showcase the versatility and applicability of our approach.

\end{description}

Regarding the paper organization, in Section \ref{sec-prelim} we present some preliminary concepts, followed by Section \ref{sec-width}, where Specific Objectives 1-3 are addressed. In Section \ref{sec-w-iv-overlaps}, we encompass Specific Objectives 4-7, with the final remarks being presented in Section \ref{sec-conclusion}. 

\section{Preliminaries} \label{sec-prelim}

In this section, we recall some concepts on (ultramodular) aggregation functions, overlap functions, interval mathematics, admissible orders  and (admissibly ordered) interval-valued overlap functions.

\subsection{Aggregation Functions}

\begin{definition}\cite{KMP00}
A function $N : [0,1] \rightarrow [0,1]$ is a fuzzy negation if the following conditions hold:
\begin{description}
\item \textbf{(N1)} $N(0) = 1$ and $N(1) = 0$;
\item \textbf{(N2)} If $x \leq y$ then $N(y) \leq N(x)$, for all $x,y \in [0,1]$.
\end{description}

If $N$ also satisfies the involutive property,
\begin{description}
\item \textbf{(N3)} $N(N(x)) = x$, for all $x \in [0,1]$,
\end{description}
then it is said to be a strong fuzzy negation.
\end{definition}
\begin{example}
The Zadeh negation given by
\begin{equation*}
N_Z(x) = 1 - x,
\end{equation*}
is a strong fuzzy negation.
\end{example}

\begin{definition}\cite{KMP00}
Given a fuzzy negation $N: [0,1] \rightarrow [0,1]$ and a function $F: [0,1]^2 \rightarrow [0,1]$, then the function $F^N: [0,1]^2 \rightarrow [0,1]$ defined, for all $x,y \in [0,1]$, by
\begin{equation}\label{eq-dual}
F^N(x,y) = N(F(N(x),N(y))),
\end{equation}
is the $N$-dual of $F$.
\end{definition}
%
%

When it is clear by the context, the $N_Z$-dual function (dual with respect to the Zadeh negation) of $F$ will be just called dual of $F$, and will be denoted by $F^d$.

\begin{definition} \cite{BeliakovBustinceCalvo_book}\label{def_aggreg}
An aggregation function is any function $A: [0, 1]^n \rightarrow [0, 1]$ that satisfies the following conditions:
\begin{description}
\item [(A1)] $A$ is increasing in each argument;
\item [(A2)] $A(0, \ldots, 0) = 0$ and $A(1, \ldots, 1) = 1$.
\end{description}
\end{definition}

\begin{example}
For $\alpha \in [0,1]$, the mapping $K_\alpha: [0,1]^2 \rightarrow [0,1]$, defined, for all $x,y \in [0,1]$, by
\begin{equation}\label{eq-kalpha}
K_\alpha(x,y)= x + \alpha \cdot (y - x),
\end{equation}
is an aggregation function.
\end{example}

In \cite{BeliakovBustinceCalvo_book}, one may find the concepts of conjunctive and disjunctive aggregation function. In this paper, we need a more general definition:

\begin{definition}
Consider a function $F: [0,1]^2 \rightarrow [0,1]$. Then, $F$ is said to be:
\begin{description}
  \item[a)] Conjunctive, if $F(x,y) \leq \min\{x,y\}$ for all $x,y \in [0,1]$;
  \item[b)] Disjunctive, if $F(x,y) \geq \max\{x,y\}$ for all $x,y \in [0,1]$.
\end{description}
\end{definition}

The definition of ultramodular aggregation functions is a key concept in this work:

\begin{definition} \cite{KLEMENT-ultra} \label{def-ultra}
An aggregation function $A: [0,1]^2 \rightarrow [0,1]$ is called ultramodular if, for all $x_1,x_2,y_1,y_2,\epsilon,\delta \in [0,1]$, such that $x_2 + \epsilon, y_2 + \delta \in [0,1]$, $x_1 \leq x_2$ and $y_1 \leq y_2$, it holds that:
\begin{equation}\label{eq-ultra}
A(x_1 + \epsilon, y_1 + \delta) - A(x_1,y_1) \leq A(x_2 + \epsilon, y_2 + \delta) - A(x_2,y_2).
\end{equation}
\end{definition}

\begin{proposition}\label{prop-der}  \cite{KLEMENT-ultra}
Assume that all partial derivatives of order 2 of the aggregation function $A : [0,1]^2 \rightarrow [0,1]$ exist. Then $A$ is
ultramodular if and only if all partial derivatives of order 2 are non-negative.
\end{proposition}

\begin{theorem} \cite{KLEMENT-ultra} \label{theo-compultra}
Let $A_1, A_2, A_3: [0,1]^2 \rightarrow [0,1]$ be ultramodular aggregation functions.  Then, the composite function $A: [0,1]^2 \rightarrow [0,1]$ given, for all $x,y \in [0,1]$, by $A(x,y) = A_3(A_1(x,y),A_2(x,y))$ is an ultramodular aggregation function.
\end{theorem}

\begin{corollary} \cite{KLEMENT-ultra} \label{coro-convexsumultra}
Let $A_1,A_2 : [0,1]^2 \rightarrow[0,1]$ be ultramodular aggregation functions and
$K_{\alpha}: [0,1]^2 \rightarrow [0,1]$ as defined in Equation (\ref{eq-kalpha}) . Then, we have that the function $A_{\alpha}: [0,1]^2 \rightarrow[0,1]$ given, for all $x,y,\alpha \in [0,1]$, by $A_{\alpha}(x,y) = K_{\alpha}(A_1(x,y),A_2(x,y))$, is an ultramodular aggregation function.
\end{corollary}

\begin{example}The following are examples of ultramodular aggregation functions:
\begin{description}
  \item[1)] The weighted sum $K_{\alpha}$, as defined in Equation (\ref{eq-kalpha});
  \item[2)] The product overlap (see Table \ref{table:exoverlaps}).
\end{description}
\end{example}
By Propositions 2.2 and 2.7 in \cite{KLEMENT-ultra}, the following result is immediate.
\begin{proposition}\label{prop-moderateg}
Let $A: [0,1]^2 \rightarrow [0,1]$ be an ultramodular aggregation function. Then, it holds that:
\begin{equation*}
A(x^*,y) + A(x,y^*) \leq A(x^*,y^*)+A(x,y),
\end{equation*}
for all $x^*,y^*,x,y \in [0,1]$ such that $x \leq x^*$ and $y \leq y^*$.
\end{proposition}

 Observe that Proposition \ref{prop-moderateg} also follows directly from Definition \ref{def-ultra}, by taking $x_1=x_2=x$, $y_1=y$, $y_2=y^*$, $\epsilon = x^*-x$ and $\delta=0$.

\begin{definition} ~\cite{Bus10a, Bedregal-overlap}
		\label{defOverlap} An overlap function is any bivariate function $O:[0,1]^2\rightarrow [0,1]$ that satisfies the following conditions,  for all $x,y \in [0,1]$:
\begin{description}
 \item [(O1)]   $O$  is commutative;
 \item [(O2)]	$O(x,y)=0$ if and only if $xy =0$;
  \item [(O3)] $O(x,y)=1$ if and only if $xy = 1$;
  \item [(O4)]  $O$ is increasing;
  \item [(O5)] $O$ is continuous.	\end{description}	
	\end{definition}

Note that an overlap function is, in particular, an aggregation function. If for all  $x,y,z \in (0,1]$ one has that $x<y \Leftrightarrow O(x,z)<O(y,z)$, then $O$ is called a \emph{strict overlap function}.

By Theorem 4 in \cite{Bus10a}, one has that:

\begin{proposition} \label{prop-compov}
Let $O_1, O_2, O_3: [0,1]^2 \rightarrow [0,1]$ be overlap functions.  Then, the composite function $O_C: [0,1]^2 \rightarrow [0,1]$ given, for all $x,y \in [0,1]$ by $O_C(x,y) = O_3(O_1(x,y),O_2(x,y))$ is an overlap function.
\end{proposition}
\begin{proposition}\cite{Bus10a} \label{prop-convexsumov}
Let $O_1,O_2 : [0,1]^2 \rightarrow[0,1]$ be overlap functions. Then, we have that function $O_{\alpha}: [0,1]^2 \rightarrow[0,1]$ given, for all $x,y,\alpha \in [0,1]$, by $O_{\alpha}(x,y) = K_{\alpha}(O_1(x,y),O_2(x,y))$ is an overlap function.
\end{proposition}

\begin{table}[t]
	\caption{Examples of overlap functions}		\label{table:exoverlaps}
	\centering
	\begin{tabular}{ll}
		\hline
        Name & Definition\\
		\hline 
		& \\[-0.2cm]
        Product & $O_P(x,y) = x \cdot y$\\[0.2cm]

		Minimum & $O_{M}(x,y) = \min\{x,y\}$\\[0.2cm]

        Geom. Mean & $O_{Gm}(x,y) = \sqrt{x \cdot y}$ \\[0.2cm]

        OmM Overlap & $O_{M}(x,y) = \min\{x,y\} \cdot \max\{x^2,y^2\}$\\[0.2cm]

        OB Overlap & $O_{OB}(x,y) = \min\{x\sqrt{y},y\sqrt{x}\} $ \\[0.2cm]
		
		Ot Overlap & $O_{t}(x,y)=\frac{(2x-1)^3 + 1}{2} \cdot \frac{(2y-1)^3 + 1}{2}$ \\[0.1cm]
\hline 
		
	\end{tabular}
\end{table}

For properties of overlap functions and related concepts, see also \cite{8118195,8632733,QIAO20181a,ZHOU2019,ZHANG2019,QIAO2018,QIAO20181,DEMIGUEL2018,additive-generators-FSS,Dimuro201439,Dimuro-QL}.

\subsection{Interval Mathematics}

Let us denote as $L([0,1])$ the set of all closed subintervals of the unit interval $[0,1]$. Given any $X = [x_1, x_2] \in L([0,1])$, $\underline{X}=x_1$ and $\overline{X}=x_2$ denote, respectively, the left and right projections of $X$, and $w(X) = \overline{X} - \underline{X}$ denotes the width of $X$. When $\underline{X} = \overline{X}$, and consequently $w(X) = 0$, we call $X$ a degenerate interval.

The interval product is defined, for all $X,Y \in L([0,1])$, by:
\begin{equation*}
X\cdot Y=[\underline{X}\cdot \underline{Y}, \overline{X} \cdot \overline{Y}].
\end{equation*}
The product and inclusion partial orders are defined for all $X,Y \in L([0,1])$, respectively, by \cite{moore2009}:
\begin{eqnarray*}
X\leq_{Pr} Y &\Leftrightarrow& \underline{X} \leq \underline{Y} \, \wedge \,  \overline{X} \leq \overline{Y};\\
X \subseteq Y &\Leftrightarrow& \underline{X} \geq \underline{Y} \, \wedge \,  \overline{X} \leq \overline{Y}.
\end{eqnarray*}

We call as $\leq_{Pr}$-increasing a function that is increasing with respect to the product order $\leq_{Pr}$. The projections $IF^-, IF^+:[0,1]^2 \rightarrow [0,1]$ of $IF: L([0,1])^2 \rightarrow L([0,1])$ are defined, respectively, by:
\begin{eqnarray}
IF^-(x,y) &=& \underline{IF([x, x], [y, y])}; \label{eq-menos}\\
IF^+(x,y) &= &\overline{IF([x, x], [y, y])} \label{eq-mais}.
\end{eqnarray}

Given two increasing functions $F,G: [0,1]^2 \rightarrow [0,1]$ such that $F \leq G$, we define the function $\widehat{F,G}: L([0,1])^2 \rightarrow L([0,1])$ as
\begin{equation}\label{eq-rep-pont}\widehat{F,G}(X,Y) =  [F(\underline{X},\underline{Y}),G(\overline{X},\overline{Y})].\end{equation}

An interval-valued function $IF$ is said to be Moore-continuous if it is continuous with respect to the Moore metric \cite{moore2009} $d_M: L([0,1])^2 \rightarrow \mathbb{R}$, defined, for all $X,Y \in L([0,1])$, by:
\begin{equation*}
  d_M(X,Y) = \max(|\underline{X} - \underline{Y}|,|\overline{X} - \overline{Y}|).
\end{equation*}

\begin{definition}\label{representable} \cite{Dimuro20113898}
	Let $IF: L([0,1])^2 \rightarrow L([0,1])$ be an $\leq_{Pr}$-increasing interval function. $IF$ is said to be representable if there exist increasing functions $F, G: [0,1]^2 \rightarrow [0,1]$ such that $F \leq G$ and $F = \widehat{F,G}$.
\end{definition}

The functions $F$ and $G$ are the \emph{representatives} of the interval function $IF$. When $IF = \widehat{F,F}$, we denote simply as $\widehat{F}$. In this case, $IF$ is said to be the best interval representation of $F$, as in \cite{BEDREGAL20101373,Dimuro20113898}.

Consider $\alpha \in [0,1]$ and the aggregation function $K_{\alpha}$ as defined in Equation (\ref{eq-kalpha}). Then, given an interval $X \in L([0,1])$, we denote $K_{\alpha}(\underline{X}, \overline{X})$ simply as $K_{\alpha}(X)$. Also, it is immediate that
\begin{equation}\label{int-kalfa-w}
[K_{\alpha}(X) - \alpha \cdot w(X), K_{\alpha}(X) + (1 - \alpha) \cdot w(X)] = X,
\end{equation}
for all $\alpha \in [0,1]$.

\subsection{Admissible Orders}

The notion of admissible orders for intervals came from the interest in extending the product order $\leq_{Pr}$ to a total order.

\begin{definition}\cite{BUSTINCE201369}
	Let $(L([0,1]),\leq_{AD})$ be a partially ordered set. The order $\leq_{AD}$ is called an admissible order if
	\begin{description}
		\item[(i)] $\leq_{AD}$ is a total order on $L([0,1])$;
		\item[(ii)] For all $X, Y \in L([0, 1])$, $X \leq_{AD} Y$ whenever $X \leq_{Pr} Y$.
	\end{description}
	
	\end{definition}

In other words, an order $\leq_{AD}$ on $L([0, 1])$ is admissible, if it is total and refines the order $\leq_{Pr}$ \cite{BUSTINCE201369}.

\begin{proposition}\cite{BUSTINCE201369}\label{adm-geradoras}
Let $A_1,A_2: [0,1]^2 \rightarrow [0,1]$ be two continuous aggregation functions, such that, for all $X,Y \in L([0,1])$, the equalities $A_1(\underline{X}, \overline{X})=A_1(\underline{Y}, \overline{Y})$ and $A_2(\underline{X}, \overline{X})=A_2(\underline{Y}, \overline{Y})$ can hold only if $X=Y$. Define the relation $\leq_{A_1,A_2}$ on $L([0,1])$ by
\begin{eqnarray*}
&&X \leq_{A_1,A_2} Y \Leftrightarrow A_1(\underline{X}, \overline{X})<A_1(\underline{Y}, \overline{Y}) \, \, \mbox{or}\\
&& \hspace{0.7cm}(A_1(\underline{X}, \overline{X})=A_1(\underline{Y}, \overline{Y}) \, \, \mbox{and} \, \, A_2(\underline{X}, \overline{X})\leq A_2(\underline{Y}, \overline{Y})).
\end{eqnarray*}
Then $\leq_{A_1,A_2}$ is an admissible order on $L([0, 1])$.
\end{proposition}

The pair $(A_1,A_2)$ of aggregation functions that generates the order $\leq_{A_1,A_2}$ in Proposition \ref{adm-geradoras} is called an admissible pair of aggregation functions \cite{BUSTINCE201369}.

\begin{definition}\cite{BUSTINCE201369} \label{def-alfabeta}
For $\alpha, \beta \in [0,1]$ such that $\alpha \neq \beta$, the relation $\leq_{\alpha,\beta}$ is defined by
\begin{eqnarray*}
&&X \leq_{\alpha,\beta} Y \Leftrightarrow K_\alpha(\underline{X},\overline{X})<K_\alpha(\underline{Y},\overline{Y}) \, \, \mbox{or} \\ &&(K_\alpha(\underline{X},\overline{X})=K_\alpha(\underline{Y},\overline{Y}) \, \, \mbox{and} \, \, K_\beta(\underline{X},\overline{X}) \leq K_\beta(\underline{Y},\overline{Y})).
\end{eqnarray*}
\end{definition}

Then, the relation $\leq_{\alpha,\beta}$ is the admissible order generated by the admissible pair of aggregation functions $(K_\alpha, K_\beta)$, that is, $\leq_{\alpha,\beta} = \leq_{K_\alpha,K_\beta}$ \cite{BUSTINCE201369}.

\begin{lemma} \cite{BUSTINCE201369}\label{lemma-novo} For any $\alpha, \beta \in [0,1]$, $\alpha \neq \beta$, it holds that: (i)$\beta>\alpha \Rightarrow \leq_{\alpha,\beta}=\leq_{\alpha,1}$; (ii) $\beta<\alpha \Rightarrow \leq_{\alpha,\beta}=\leq_{\alpha,0}$.
\end{lemma}

%

\subsection{Interval-valued Overlap Functions}

\begin{definition}\cite{Marichal2002} \label{def_aggreg-int}
	An interval-valued function $IA: L([0,1])^2\rightarrow L([0,1])$ is said to be an interval-valued aggregation function if the following conditions hold:
\begin{description}
	\item[(IA1)] $IA$ is $\leq_{Pr}$-increasing;
	\item[(IA2)] $IA([0,0],[0,0])=[0,0]$ and $IA([1,1],[1,1])=[1,1]$.
\end{description}
\end{definition}

\begin{definition}\label{def-int-overlap}\cite{Qiao201719,BEDREGAL20171} An interval-valued (iv) overlap function  is a mapping  $IO:L([0,1])^2\rightarrow L([0,1])$ that respects the following conditions:
\begin{description}
 \item [(IO1)]   $IO$  is commutative;

\item [(IO2)] $IO (X,Y)=[0,0]$ if and only if $X \cdot Y =[0,0]$;

\item [(IO3)] $IO(X,Y)=[1,1]$ if and only if $X \cdot Y= [1,1]$;

\item [(IO4)] $IO$ is $\leq_{Pr}$-increasing in the first component:  $IO(Y,X)\leq_{Pr} IO(Z,X)$ when $Y\leq_{Pr} Z$.

\item [(IO5)] $IO$ is Moore continuous. 
\end{description}
\end{definition}

Note that, by \textbf{(IO1)} and \textbf{(IO4)}, iv-overlap functions are also monotonic in the second component.

An iv-overlap function $IO: L([0,1])^2 \rightarrow L([0,1])$ is said to be $o$-representable \cite{ASMUS2020} if there exist overlap functions  $O_1, O_2: [0,1]^2 \rightarrow [0,1]$ such that $O_1 \leq O_2 $ and $IO = \widehat{O_1,O_2}$.

\begin{definition}\cite{ASMUS2020-TFS}\label{def-int-AD-n-overlap} A function $AO:L([0,1])^2\rightarrow L([0,1])$ is an admissibly ordered interval-valued  overlap function for an admissible order $\leq_{AD}$ ($\leq_{AD}$-overlap function) if it satisfies the conditions \textbf{(IO1)}, \textbf{(IO2)} and \textbf{(IO3)} of Definition \ref{def-int-overlap} and, for all $X, Y, Z \in L([0,1])$:
	
	\begin{description}	
		\item [(AO4)] $AO$ is $\leq_{AD}$-increasing:  $X \leq_{AD} Y \Rightarrow AO(X,Z) \leq_{AD} AO(Y,Z)$.
	\end{description}

\end{definition}

The following construction method for admissibly ordered interval-valued overlap functions preserves the minimal width of the input intervals:

\begin{theorem} \cite{ASMUS2020-TFS} \label{theo-benja}
	Let $O$ be a strict overlap function and $\alpha \in (0,1), \beta\in [0,1]$ such that $\alpha\neq\beta$. Then $AO^{\alpha}: L([0,1])^2\rightarrow L([0,1])$ defined, for all $X,Y \in L([0,1])$, by
\begin{eqnarray}\label{eq-benjaprelim}
AO^{\alpha}(X,Y)=[O(K_{\alpha}(X),K_{\alpha}(Y))-\alpha m, O(K_{\alpha}(X), K_{\alpha}(Y))+(1-\alpha)m],
\end{eqnarray}
 where
\begin{eqnarray*}
m = \min\{w({X}),w({Y}), O(K_{\alpha}(X),K_{\alpha}(Y)), 1 - O(K_{\alpha}(X),K_{\alpha}(Y))\}
\end{eqnarray*}
 is an $\leq_{\alpha,\beta}$-overlap function.
\end{theorem}

\section{Width-limited Interval-valued Functions} \label{sec-width}

As a motivation to the developments presented in this section, we pose the following question: with respect to a given interval-valued function, how can the width of the output interval be affected by the widths of the input intervals? In order to aid on such discussion, concerning the uncertainty propagation control in aggregation processes, we introduce the following definition:

\begin{definition}\label{def-width-limiting}
Consider an interval-valued function $IF: L([0,1])^2 \rightarrow L([0,1])$ and a mapping $B: [0,1]^2 \rightarrow [0,1]$. Then, $IF$ is said to be width-limited by $B$ if $w(IF(X,Y)) \leq B(w(X),w(Y))$, for all $X,Y \in L([0,1])$. $B$ is called a width-limiting function of $IF$.
\end{definition}

\begin{remark}
Every function $IF: L([0,1])^2 \rightarrow L([0,1])$ is width-limited by the function $B_1: [0,1]^2 \rightarrow [0,1]$ defined by $B_1(x,y) = 1$, for all $x,y \in [0,1]$.
\end{remark}



In the following, denote:
\begin{equation*}
\mathcal{IF} = \{IF: L([0,1])^2 \rightarrow L([0,1]) \, | \, IF \, \, \mbox{is a binary interval-valued function}\}
\end{equation*}
and
\begin{equation*}
\mathcal{F} = \{F: [0,1]^2 \rightarrow [0,1] \, | \, F \, \, \mbox{is binary function}\}.
\end{equation*}
First, we analyze how to obtain the least width-limiting function for a given interval-valued function:

\begin{theorem}\label{theo-geradorw}
The mapping $\mathfrak{L}: \mathcal{IF} \rightarrow \mathcal{F}$ defined for all $IF \in \mathcal{IF}$ and $\epsilon,\delta \in [0,1]$, by
\begin{equation*}\label{eq-geradorw}
\mathfrak{L}(IF)(\epsilon,\delta) = \sup_{\scriptsize{\begin{array}{c}u \in [0,1-\epsilon]\\v \in [0,1-\delta]\end{array}}}\{w(IF([u,u+\epsilon],[v,v+\delta]))\}
\end{equation*}
provides the least width-limiting function $\mathfrak{L}(IF): [0,1]^2 \rightarrow [0,1]$ for $IF$.
\end{theorem}
\begin{proof}
It is clear that $\mathfrak{L}(IF)$ is well defined, since
\begin{equation*}
\sup_{\scriptsize{\begin{array}{c}u \in [0,1-\epsilon]\\v \in [0,1-\delta]\end{array}}}\{w(IF([u,u+\epsilon],[v,v+\delta]))\} \in [0,1],
\end{equation*}
for all $IF \in \mathcal{IF}$ and all $\epsilon,\delta \in [0,1]$.

Now, observe that
 \begin{eqnarray*}
\mathfrak{L}(IF)(\epsilon,\delta) & = & \sup_{\scriptsize{\begin{array}{c}u \in [0,1-\epsilon]\\v \in [0,1-\delta]\end{array}}}\{w(IF([u,u+\epsilon],[v,v+\delta]))\}\\
&\geq& w(IF([u,u+\epsilon],[v,v+\delta]))
\end{eqnarray*}
for all $u \in [0,1-\epsilon],v \in [0, 1-\delta]$, showing that $IF$ is width-limited by $\mathfrak{L}(IF)$, since $w([u,u+\epsilon])=\epsilon$ and $w([v,v+\delta])=\delta$.

Finally, suppose that there exists a function $B: [0,1]^2 \rightarrow [0,1]$ such that: (i) $B$ is a width-limiting function for $IF$; (ii) there exist $\epsilon_0,\delta_0 \in [0,1]$ such that $B(\epsilon_0,\delta_0) <  \mathfrak{L}(IF)(\epsilon_0,\delta_0)$. So, it follows that
  \begin{eqnarray*}
B(\epsilon_0,\delta_0) &<& \sup_{\scriptsize{\begin{array}{c}u \in [0,1-\epsilon_0]\\v \in [0,1-\delta_0]\end{array}}}\{w(IF([u,u+\epsilon_0],[v,v+\delta_0]))\}.
\end{eqnarray*}
Then, there exist $u_0 \in [u,u+\epsilon_0]$, $v_0 \in [v,v+\delta_0]$ such that
\begin{equation*}
B(\epsilon_0,\delta_0) < w(IF([u_0,u_0+\epsilon_0],[v_0,v_0+\delta_0])),
\end{equation*}
meaning that $IF$ cannot be width-limited by $B$, which is a contradiction. The conclusion is that $\mathfrak{L}(IF)$ is the least function that is width-limiting for $IF$.
\qed
\end{proof}

In the following, denote:
\begin{equation*}
\mathcal{A} = \{A: [0,1]^2 \rightarrow [0,1] \, | \, A \, \, \mbox{is an aggregation function}\}
\end{equation*}
 and
\begin{equation*}
 \mathcal{IA} = \{IA: L([0,1])^2 \rightarrow L([0,1]) \, | \, IA \, \, \mbox{is the best interval representation of an aggregation function} \, A \in \mathcal{A}\}.
\end{equation*}
  Then, a similar approach of Theorem \ref{theo-geradorw} can be used to obtain the least width-liming aggregation function for a given representable interval-valued aggregation function.

\begin{theorem}\label{theo-geradorwAG}
The mapping $\mathfrak{L}: \mathcal{IA} \rightarrow \mathcal{F}$ defined for all $IA \in \mathcal{IA}$ and $\epsilon,\delta \in [0,1]$, by
\begin{equation}\label{eq-geradorwAG}
\mathfrak{L}(IA)(\epsilon,\delta) = \sup_{\scriptsize{\begin{array}{c}u \in [0,1-\epsilon]\\v \in [0,1-\delta]\end{array}}}\{w(IA([u,u+\epsilon],[v,v+\delta]))\}
\end{equation}
provides the least width-limiting function $\mathfrak{L}(IA): [0,1]^2 \rightarrow [0,1]$ for $IA$. Moreover, $\mathfrak{L}(IA)$ is an aggregation function.
\end{theorem}
\begin{proof}
From Theorem \ref{theo-geradorw}, it only remains to be shown that $\mathfrak{L}(IA)$ respects the conditions for it to be an aggregation function, for all $IA \in \mathcal{IA}$:

\begin{description}
\item[(A1)] Consider $\epsilon_1,\epsilon_2, \delta_1, \delta_2 \in [0,1]$ such that $\epsilon_1 \leq \epsilon_2$ and $\delta_1 \leq \delta_2$. Thus, for all $u \in [0, 1 - \epsilon_2]$ and $v \in [0, 1 - \delta_2]$, it holds that
     \begin{equation*}
    [u, u + \epsilon_1] \leq [u, u + \epsilon_2] \, \, \mbox{and} \, \, [v, v + \delta_1] \leq [v, v + \delta_2].
     \end{equation*}
Since $IA$ is $\leq_{Pr}$-increasing, for all $u \in [0, 1 - \epsilon_2]$ and $v \in [0, 1 - \delta_2]$, it follows that
 \begin{equation}\label{eq-IA}
   IA([u, u + \epsilon_1],[v, v + \delta_1])  \leq_{Pr} IA([u, u + \epsilon_2],[v, v + \delta_2]).
     \end{equation}

As $IA \in \mathcal{IA}$, then there exist an aggregation function $A: [0,1]^2 \rightarrow [0,1]$ such that
\begin{equation*}
IA(X,Y)=[A(\underline{X},\underline{Y}),A(\overline{X},\overline{Y})],
\end{equation*}
for all $X,Y \in L([0,1])$. Thus, by Equation (\ref{eq-IA}), one has that 
\begin{eqnarray*}
\lefteqn{[A(u,v),A(u+\epsilon_1,v+\delta_1)] \leq_{Pr} [A(u,v),A(u+\epsilon_2,v+\delta_2)]} &&\\
& \Rightarrow & A(u+\epsilon_1,v+\delta_1) - A(u,v) \leq A(u+\epsilon_2,v+\delta_2) - A(u,v)\\
& \Rightarrow & w([A(u,v),A(u+\epsilon_1,v+\delta_1)]) \leq w([A(u,v),A(u+\epsilon_2,v+\delta_2)])\\
& \Rightarrow & w(IA([u,u+\epsilon_1],[v,v+\delta_1])) \leq w(IA([u,u+\epsilon_2],[v,v+\delta_2]))\\
& \Rightarrow & \sup_{\scriptsize{\begin{array}{c}u \in [0,1-\epsilon_1]\\v \in [0,1-\delta_1]\end{array}}}\{w(IA([u,u+\epsilon_1],[v,v+\delta_1]))\} \leq \sup_{\scriptsize{\begin{array}{c}u \in [0,1-\epsilon_2]\\v \in [0,1-\delta_2]\end{array}}}\{w(IA([u,u+\epsilon_2],[v,v+\delta_2]))\}\\
& \Rightarrow & \mathfrak{L}(IA)(\epsilon_1,\delta_1) \leq \mathfrak{L}(IA)(\epsilon_2,\delta_2),
\end{eqnarray*}
showing that $\mathfrak{L}(IA)$ is increasing.

  \item[(A2)] As $IA \in \mathcal{IA}$, it follows that
  \begin{equation*}
\mathfrak{L}(IA)(0,0) = \sup_{u,v \in [0,1]}\{w(IA([u,u],[v,v]))\}= \sup_{u,v \in [0,1]}\{A(u,v) - A(u,v)\} = 0,
\end{equation*}
and
\begin{equation*}
\mathfrak{L}(IA)(1,1) = w(IA([0,1],[0,1]))= A(1,1)-A(0,0)=1.
\end{equation*}
\end{description}
\qed
\end{proof}

\begin{example}
Let $A: [0,1]^2 \rightarrow [0,1]$ be an aggregation function defined, for all $x,y \in [0,1]$, by $A(x,y)=\frac{x+y+x \cdot y}{3}$. Then, the mapping $\mathfrak{L}(\widehat{A}): [0,1]^2 \rightarrow [0,1]$ defined, for all $\epsilon,\delta \in [0,1]$, by
\begin{eqnarray*}
\mathfrak{L}(\widehat{A})(\epsilon,\delta) & = & \sup_{\scriptsize{\begin{array}{c}u \in [0,1-\epsilon]\\v \in [0,1-\delta]\end{array}}}\{w(\widehat{A}([u,u+\epsilon],[v,v+\delta]))\}\\
& = & \sup_{\scriptsize{\begin{array}{c}u \in [0,1-\epsilon]\\v \in [0,1-\delta]\end{array}}}\{A(u+\epsilon,v+\delta) - A(u,v)\}\\
& = & \sup_{\scriptsize{\begin{array}{c}u \in [0,1-\epsilon]\\v \in [0,1-\delta]\end{array}}}\left\{\frac{u+\epsilon+v+\delta+(u+\epsilon) \cdot (v+\delta)}{3} - \left(\frac{u+v+y \cdot v}{3}\right)\right\}\\
& = & \sup_{\scriptsize{\begin{array}{c}u \in [0,1-\epsilon]\\v \in [0,1-\delta]\end{array}}}\left\{\frac{\epsilon + \delta + u \cdot \delta + \epsilon \cdot v + \epsilon \cdot \delta}{3}\right\}\\
& = & \frac{\epsilon + \delta + (1 - \epsilon) \cdot \delta + \epsilon \cdot (1-\delta) + \epsilon \cdot \delta}{3}\\
& = &\frac{2\epsilon + 2\delta -  \epsilon \cdot \delta}{3}
\end{eqnarray*}
is the least width-limiting function for $\widehat{A}$. Observe that $\mathfrak{L}(\widehat{A})$ is an aggregation function.
\end{example}

Based on the concept of ultramodularity, let us define a less restrictive extension of one-dimension convexity for bivariate aggregation functions:

\begin{definition}\label{def-abultra}
Consider $a,b \in [0,1]$. An aggregation function $A: [0,1]^2 \rightarrow [0,1]$ is called $(a,b)$-ultramodular if, for all $x,y,\epsilon,\delta \in [0,1]$ and $x + \epsilon, y + \delta, a - \epsilon, b - \delta \in [0,1]$, it holds that:
\begin{equation}\label{eq-abultra}
A(x + \epsilon, y + \delta) - A(x,y) \leq A(a, b) - A(a - \epsilon, b - \delta).
\end{equation}
\end{definition}
\begin{proposition}\label{prop-alltultra}
Let $A: [0,1]^2 \rightarrow [0,1]$ be an ultramodular aggregation function. Then, $A$ is an $(1,1)$-ultramodular aggregation function.
\end{proposition}
\begin{proof}
Immediate, since Equation (\ref{eq-abultra}), with $a = b = 1$, is a particular case of Equation (\ref{eq-ultra}) when $\epsilon + x_2 = 1$ and $\delta + y_2 = 1$.
\qed
\end{proof}

\begin{remark}
If an aggregation function $A: [0,1]^2 \rightarrow [0,1]$ is $(1,1)$-ultramodular, then, for all $x, y,\epsilon , \delta \in [0,1]$ such that $x + \epsilon, y + \delta, a - \epsilon, b - \delta \in [0,1]$, it holds that:
\begin{equation}\label{eq-wultra}
A(x + \epsilon, y + \delta) - A(x,y) \leq A^d(\epsilon, \delta),
\end{equation}
where $A^d$ is the dual of $A$.
\end{remark}

\begin{remark}\label{rem-wu}
From Proposition \ref{prop-alltultra}, we have that every ultramodular function is also $(1,1)$-ultramodular. However, the converse may not hold. For example, the Ot overlap (Table {\ref{table:exoverlaps}}) given by $O_{t}(x,y)=\frac{(2x-1)^3 + 1}{2} \cdot \frac{(2y-1)^3 + 1}{2}$, for all $x,y \in [0,1]$, is an $(1,1)$-ultramodular function. However, by Proposition \ref{prop-der}, $O_{t}$ is clearly not an ultramodular aggregation function.
\end{remark}

Now, let us present a characterization for the least width-limiting function of the best interval representation of an $(1,1)$-ultramodular aggregation function, or the best interval representation of its dual:

\begin{theorem}\label{theo-ultradual}
Let $A: [0,1]^2 \rightarrow [0,1]$ be an aggregation function, $\mathfrak{L}(\widehat{A}),\mathfrak{L}(\widehat{A^d}): [0,1]^2 \rightarrow [0,1]$ be the least width-limiting functions for $\widehat{A}$ and $\widehat{A^d}$, respectively. Then, $\mathfrak{L}(\widehat{A}) =  \mathfrak{L}(\widehat{A^d}) = A^d$ if and only if $A$ is an $(1,1)$-ultramodular aggregation function.
\end{theorem}

\begin{proof}
($\Rightarrow$) Suppose that $\mathfrak{L}(\widehat{A}) =  \mathfrak{L}(\widehat{A^d}) = A^d$. Then, we have that:
\begin{eqnarray}\label{eq-aultra}
\nonumber \lefteqn{\mathfrak{L}(\widehat{A}) = A^d}&&\\
\nonumber & \Rightarrow & \sup_{\scriptsize{\begin{array}{c}u \in [0,1-\epsilon]\\v \in [0,1-\delta]\end{array}}}\{w(\widehat{A}([u,u+\epsilon],[v,v+\delta]))\} = 1 - A(1-\epsilon,1-\delta)\\
& \Rightarrow & A(u + \epsilon, v + \delta) - A(u,v) \leq A(1,1) - A(1-\epsilon,1-\delta), \, \, \mbox{for all} \, \,  u \in [0,1-\epsilon],v \in [0,1-\delta].
\end{eqnarray}
From Equation (\ref{eq-aultra}), we conclude that $A$ is $(1,1)$-ultramodular.

($\Leftarrow$)
Suppose that $A: [0,1]^2 \rightarrow [0,1]$ is an $(1,1)$-ultramodular aggregation function. Then, for all $\epsilon, \delta \in [0,1]$, it holds that:
\begin{eqnarray*}
\mathfrak{L}(\widehat{A})(\epsilon,\delta) & = & \sup_{\scriptsize{\begin{array}{c}u \in [0,1-\epsilon]\\v \in [0,1-\delta]\end{array}}}\{w(\widehat{A}([u,u+\epsilon],[v,v+\delta]))\}\\
& = & \sup_{\scriptsize{\begin{array}{c}u \in [0,1-\epsilon]\\v \in [0,1-\delta]\end{array}}}\{A(u + \epsilon, v + \delta) - A(u,v)\}\\
& = & A(1-\epsilon + \epsilon, 1-\delta + \delta) - A(1-\epsilon,1-\delta)\\
& = & 1 - A(1-\epsilon,1-\delta)\\
& = & A^d(\epsilon,\delta),
\end{eqnarray*}
and
\begin{eqnarray*}
\mathfrak{L}(\widehat{A^d})(\epsilon,\delta) & = & \sup_{\scriptsize{\begin{array}{c}u \in [0,1-\epsilon]\\v \in [0,1-\delta]\end{array}}}\{w(\widehat{A^d}([u,u+\epsilon],[v,v+\delta]))\}\\
& = & \sup_{\scriptsize{\begin{array}{c}u \in [0,1-\epsilon]\\v \in [0,1-\delta]\end{array}}}\{A(1 - u,1 - v) - A(1 - u - \epsilon, 1 - v - \delta)\}\\
& = & A(1,1) - A(1-\epsilon,1-\delta)\\
& = & A^d(\epsilon,\delta),
\end{eqnarray*}
since $A$ is $(1,1)$-ultramodular. Thus, if $A$ is an $(1,1)$-ultramodular aggregation function, then $\mathfrak{L}(\widehat{A}) = \mathfrak{L}(\widehat{A^d}) = A^d$.
\qed
\end{proof}

\begin{remark}
In the context of Theorem \ref{theo-ultradual}, as $\widehat{A}$ and $\widehat{A^d}$ are representable iv-aggregation functions, then their least width-limiting function $A^d$ is an aggregation function, as stated by Theorem \ref{theo-geradorwAG}. Also, observe that the function $A$ does not need to be ultramodular.
\end{remark}

\begin{example}
The least width-limiting function for either $\widehat{O_t}$ (the best interval representation of the overlap function $O_t$, shown in Table \ref{table:exoverlaps}) or $\widehat{O_t^d}$ (the best interval representation of the dual of $O_t$) is $O_t^d$, as $O_t$ is an $(1,1)$-ultramodular aggregation function.
\end{example}

Since every ultramodular aggregation function is also $(1,1)$-ultramodular, the following result is immediate.

\begin{corollary}\label{coro-ultradual}
Let $A: [0,1]^2 \rightarrow [0,1]$ be an aggregation function, $\mathfrak{L}(\widehat{A}),\mathfrak{L}(\widehat{A^d}): L([0,1])^2 \rightarrow L([0,1])$ be the least width-limiting functions for $\widehat{A}$ and $\widehat{A^d}$, respectively. Then, $\mathfrak{L}(\widehat{A}) =  \mathfrak{L}(\widehat{A^d}) = A^d$ if and only if $A$ is an ultramodular aggregation function.
\end{corollary}


\begin{example}\label{ex-leasta}
Here we present some examples of width-limiting functions for the best interval representation of either an ultramodular aggregation function or its dual:
\begin{description}
\item [1)] The least width-limiting function for either $\widehat{O_P}$ (the best interval representation of the product overlap) or $\widehat{O_P^d}$ (the best interval representation of the dual of $O_P$) is $O_P^d$;
\item [2)] The least width-limiting function for $\widehat{K_{\alpha}}$ (the best interval representation of the weighted sum), is $K_{\alpha}^d = K_{\alpha}$, with $\alpha \in [0,1]$;
\item [3)] Consider the aggregation function $AM: [0,1]^2 \rightarrow [0,1]$ given by $AM(x,y)=\frac{x+y}{2}$ (arithmetic mean). So, the least width-limiting function for $\widehat{AM}$ (the best interval representation of the arithmetic mean), is $AM^d = AM$.
    \end{description}

\end{example}

\begin{proposition}
 Let $IF_1,IF_2,IG,IH \in \mathcal{IA}$, such that $IH(X,Y) = IG(IF_1(X,Y),IF_2(X,Y))$, for all $X,Y \in L([0,1])$. Then, it holds that:
  \begin{eqnarray*}
\mathfrak{L}(IH) \leq \mathfrak{L}(IG)(\mathfrak{L}(IF_1),\mathfrak{L}(IF_2)).
\end{eqnarray*}
\end{proposition}
\begin{proof}


Consider $IF_1([x_1,x_1+\epsilon],[x_2,x_2+\delta]) = [y_1,y_1+\epsilon^*]$, $IF_2([x_1,x_1+\epsilon],[x_2,x_2+\delta]) = [y_2,y_2+\delta^*]$, with $\epsilon,\delta,\epsilon^*,\delta^*,\in [0,1]$ and $x_1+\epsilon, x_2+\delta, y_1+\epsilon^*, y_2+\delta^* \in [0,1]$. Then, it follows that:
\begin{eqnarray*}
\lefteqn{w(IH([x_1,x_1+\epsilon],[x_2,x_2+\delta]))}&&\\
&=& w(IG(IF_1([x_1,x_1+\epsilon],[x_2,x_2+\delta]),IF_2([x_1,x_1+\epsilon],[x_2,x_2+\delta])))\\
&=& w(IG([y_1,y_1+\epsilon^*],[y_2,y_2+\delta^*]) \\
&\leq& \mathfrak{L}(IG)(\epsilon^*,\delta^*), \, \, \mbox{by Theorem \ref{theo-geradorwAG}}\\
&\leq& \mathfrak{L}(IG)(\mathfrak{L}(IF_1)(\epsilon,\delta),\mathfrak{L}(IF_2)(\epsilon,\delta)),
\end{eqnarray*}
which means that $\mathfrak{L}(IG)(\mathfrak{L}(IF_1),\mathfrak{L}(IF_2))$ is a width-limiting function for $IH$.

However, as $\mathfrak{L}(IH)$ is the least width-limiting function for $IH$ (by Theorem \ref{theo-geradorwAG}), thus, one concludes that
\begin{equation*}
\mathfrak{L}(IH) \leq \mathfrak{L}(IG)(\mathfrak{L}(IF_1),\mathfrak{L}(IF_2)).
\end{equation*}
\qed
\end{proof}

\begin{example}
\begin{description}
\item [1)] Take $IF_1=\widehat{AM}$, $IF_2=\widehat{O_P}$, $IG = \widehat{K_{\alpha}}$, as presented in Example \ref{ex-leasta}. Then, let $IH: L([0,1])^2 \rightarrow L([0,1])$ be an iv-aggregation function defined, for all $X,Y \in L([0,1])$ with $\alpha \in [0,1]$, by
\begin{eqnarray*}
IH(X,Y) & = & \widehat{K_{\alpha}}(\widehat{AM}(X,Y),\widehat{O_P}(X,Y))\\
 & = &\widehat{K_{\alpha}}([AM(\underline{X},\underline{Y}),AM(\overline{X},\overline{Y})],[O_P(\underline{X},\underline{Y}),O_P(\overline{X},\overline{Y})]))\\
& = &\widehat{K_{\alpha}}\left(\left[\frac{\underline{X} + \underline{Y}}{2},\frac{\overline{X} + \overline{Y}}{2}\right],[\underline{X} \cdot \underline{Y},\overline{X} \cdot \overline{Y}]\right).
\end{eqnarray*}

Since $AM, O_P$ and $K_{\alpha}$ are ultramodular, it holds that:
  \begin{eqnarray*}
\lefteqn{\mathfrak{L}(IH)(\epsilon,\delta)}&&\\
   &= & \sup_{\scriptsize{\begin{array}{c}u \in [0,1-\epsilon]\\v \in [0,1-\delta]\end{array}}}\left\{K_{\alpha}\left(\frac{u + \epsilon + v + \delta}{2},(u+\epsilon) \cdot (v + \delta)\right) - K_{\alpha}\left(\frac{u+v}{2}, u \cdot v\right)\right\}\\
    &= & K_{\alpha}\left(\frac{1 - \epsilon + \epsilon + 1 - \delta + \delta}{2},(1 - \epsilon+\epsilon) \cdot (1 - \delta + \delta)\right) - K_{\alpha}\left(\frac{1 - \epsilon+1 - \delta}{2}, (1 - \epsilon) \cdot (1 - \delta)\right)\\
    &= & K_{\alpha}\left(1,1\right) - K_{\alpha}\left(\frac{(1 - \epsilon) + (1 - \delta)}{2}, (1 - \epsilon) \cdot (1 - \delta)\right)\\
    &= & 1 - K_{\alpha}\left(AM(1-\epsilon,1-\delta), O_P(1 - \epsilon, 1 - \delta)\right)\\
    &= & 1 - K_{\alpha}\left(1 - AM(\epsilon,\delta), 1 - O_P^d(\epsilon,\delta)\right)\\
     &= & K_{\alpha}\left(AM(\epsilon,\delta), O_P^d(\epsilon,\delta)\right).
\end{eqnarray*}

From Theorem \ref{theo-ultradual} we have that $\mathfrak{L}(\widehat{O_P}) = O_P^d$, $\mathfrak{L}(\widehat{AM}) = AM^d = AM$ and $\mathfrak{L}(\widehat{K_{\alpha}}) = K_{\alpha}^d = K_{\alpha}$, for all $\alpha \in [0,1]$. So, we conclude that
 \begin{eqnarray*}
\mathfrak{L}(\widehat{K_{\alpha}})(\mathfrak{L}(\widehat{AM})(\epsilon,\delta),\mathfrak{L}(\widehat{O_P})(\epsilon,\delta)) = K_{\alpha}(AM(\epsilon,\delta),O_P^d(\epsilon,\delta)) =\mathfrak{L}(IH).
\end{eqnarray*}

\item [2)]
Now, take $IF_1=\widehat{O_P}$, $IF_2=\widehat{O_P^d}$, $IG = \widehat{K_{0.25}}$, with $\alpha = 0.25$. Then, let $IH: L([0,1])^2 \rightarrow L([0,1])$ be the iv-aggregation function defined, for all $X,Y \in L([0,1])$, by
\begin{eqnarray*}
IH(X,Y) & = &\widehat{K_{0.25}}(\widehat{O_P}(X,Y),\widehat{O_P^d}(X,Y))\\
 & = &\widehat{K_{0.25}}([O_P(\underline{X},\underline{Y}),O_P(\overline{X},\overline{Y})],[O_P^d(\underline{X},\underline{Y}),O_P^d(\overline{X},\overline{Y})]))\\
& = &\widehat{K_{0.25}}\left(\left[\underline{X} \cdot \underline{Y},\overline{X} \cdot \overline{Y}\right],[\underline{X} + \underline{Y} - \underline{X} \cdot \underline{Y},\overline{X} + \overline{Y} - \overline{X} \cdot \overline{Y}]\right)\\
& = &\left[\frac{\underline{X} + \underline{Y} + 2 \cdot \underline{X} \cdot \underline{Y}}{4}, \frac{\overline{X} + \overline{Y} + 2 \cdot \overline{X} \cdot \overline{Y}}{4}\right].
\end{eqnarray*}

Thus,
  \begin{eqnarray*}
\lefteqn{\mathfrak{L}(IH)(\epsilon,\delta)}&&\\
   &= & \sup_{\scriptsize{\begin{array}{c}u \in [0,1-\epsilon]\\v \in [0,1-\delta]\end{array}}}\left\{\frac{u + \epsilon + v + \delta + 2 \cdot (u + \epsilon) \cdot (v + \delta)}{4} - \frac{u + v + 2 \cdot u \cdot v}{4} \right\}\\
   &= & \sup_{\scriptsize{\begin{array}{c}u \in [0,1-\epsilon]\\v \in [0,1-\delta]\end{array}}}\left\{\frac{\epsilon + \delta + 2\cdot u \cdot \delta + 2 \cdot \epsilon \cdot v + 2 \cdot \epsilon \cdot \delta}{4} \right\}\\
      &= & \frac{\epsilon + \delta + 2 \cdot (1 - \epsilon) \cdot \delta + 2 \cdot \epsilon \cdot (1 - \delta) + 2 \cdot \epsilon \cdot \delta}{4}\\
     &= & \frac{3\epsilon + 3\delta - 2\epsilon\delta}{4}.
\end{eqnarray*}


From Theorem \ref{theo-ultradual} we have that $\mathfrak{L}(\widehat{O_P}) = \mathfrak{L}(\widehat{O_P^d}) = O_P^d$ and $\mathfrak{L}(\widehat{K_{0.25}}) = K_{0.25}^d = K_{0.25}$. So, we have that
 \begin{eqnarray*}
\lefteqn{\mathfrak{L}(\widehat{K_{0.25}})(\mathfrak{L}(\widehat{O_P})(\epsilon,\delta),\mathfrak{L}(\widehat{O_P^d})(\epsilon,\delta))} &&\\
&=& K_{0.25}(O_P^d(\epsilon,\delta),O_P^d(\epsilon,\delta)) \\
&=& O_P^d(\epsilon,\delta)\\
& \geq & \frac{3\epsilon + 3\delta - 2\epsilon\delta}{4}\\
&=& \mathfrak{L}(IH).
\end{eqnarray*}
\end{description}
\end{example}

\begin{remark}\label{rem-def}
Consider an interval-valued function $IF: L([0,1])^2 \rightarrow L([0,1])$ and an aggregation function $A: [0,1]^2 \rightarrow [0,1]$. If $IF$ is width-limited by $A$, we have that, for any $X,Y \in L([0,1])$:
\begin{description}
  \item[1)] If $A = \max$, then $IF$ is limited by the maximal width of the input intervals $X,Y$;
  \item[2)] If $A = \min$, then $IF$ is limited by the minimal width of the input intervals $X,Y$;
  \item[3)] If $A$ is conjunctive and either $X$ or $Y$ is degenerate, then $IF(X,Y)$ is also degenerate;
  \item[4)] If $A$ is averaging, then $\min\{w(X), w(Y)\} \leq w(IF(X,Y)) \leq \max\{w(X), w(Y)\}$.
\end{description}
\end{remark}

\section{Width-limited Interval-valued Overlap Functions}\label{sec-w-iv-overlaps}

The aim of this section is to apply the newly developed concepts of width-limited interval-valued functions and width limiting functions to obtain a new definition of width-limited interval-valued overlap functions, taking into consideration different partial orders. Also, we are going to present three construction methods for width-limited interval-valued overlap functions, followed by some examples and comparisons.

First, to enable a more flexible definition of interval-valued functions, let us define the concept of increasingness with respect to a pair of partial orders:

\begin{definition}\label{def-leqleq}
Let $IF: L([0,1])^2 \rightarrow L([0,1])$ be an interval-valued function and $\leq_1$, $\leq_2$ be two partial order relations on $L([0,1])$. Then, $IF$ is said to be $(\leq_1,\leq_2)$-increasing if the following condition holds, for all $X_1,X_2,Y_1,Y_2 \in L([0,1])$:
\begin{equation*}
X_1 \leq_1 X_2 \, \, \wedge \, \, Y_1 \leq_1 Y_2 \Rightarrow IF(X_1,Y_1) \leq_2 IF(X_2,Y_2).
\end{equation*}
\end{definition}

When an interval-valued function $IF: L([0,1])^2 \rightarrow L([0,1])$ is $(\leq,\leq)$-increasing, we denote it simply as \\$\leq$-increasing, for any partial order relation $\leq$ on $L([0,1])$.


\begin{proposition}
Let  $\leq_{AD}$ be an admissible order on $L([0,1])$. Then, an $\leq_{Pr}$-increasing function $IF: L([0,1])^2 \rightarrow L([0,1])$  is also $(\leq_{Pr},\leq_{AD})$-increasing.
\end{proposition}

\begin{proof}
Immediate, as $\leq_{AD}$ is an admissible order and, as such, refines $\leq_{Pr}$. \qed
\end{proof}

\begin{example}
Given an overlap function $O: [0,1]^2 \rightarrow [0,1]$, the $\leq_{\alpha,\beta}$-overlap function $AO^{\alpha}: L([0,1])^2\rightarrow L([0,1])$ defined in Equation (\ref{eq-benjaprelim}) (Theorem \ref{theo-benja}) is $(\leq_{Pr},\leq_{\alpha,\beta})$-increasing for all $\alpha,\beta \in [0,1]$ such that $\alpha \neq \beta$.
\end{example}

Here, we present the definition of width-limited interval-valued overlap functions:

\begin{definition}\label{def-w-iv-ov}

Let $B: [0,1]^2 \rightarrow [0,1]$ be a commutative and increasing function and $\leq_1$, $\leq_2$ be two partial order relations on $L([0,1])$. Then, the mapping $IOw: L([0,1])^2 \rightarrow L([0,1])$ is said to be a width-limited interval-valued overlap function (w-iv-overlap function) with respect to the tuple ($\leq_1$, $\leq_2$, $B$), if the following conditions hold for all $X,Y \in L([0,1])$:

\begin{description}
	\item[(IOw1)] $IOw$ is commutative;
	\item[(IOw2)] $IOw(X,Y) = [0,0] \Leftrightarrow X \cdot Y = [0,0]$;
	\item[(IOw3)] $IOw(X,Y) = [1,1] \Leftrightarrow X \cdot Y = [1,1]$;
	\item[(IOw4)] $IOw$ is $(\leq_1,\leq_2)$-increasing;
    \item[(IOw5)] $IOw$ is width-limited by B.
\end{description}
\end{definition}

\begin{remark}
Taking a similar approach as in \cite{ASMUS2020-TFS} when defining admissibly ordered interval-valued overlap functions, we do not require the continuity as a condition in Definition \ref{def-w-iv-ov}. The original definition of overlap functions (Defition \ref{def-int-overlap}) included the Moore continuity as a necessary condition as the goal was to be applied in image processing problems \cite{Bus10a}, which is not the case here.
\end{remark}

Now, let us presents some results regarding width-limited interval-valued overlap functions obtained through the best interval representation of an overlap function:

\begin{proposition}
Let $O: [0,1]^2 \rightarrow [0,1]$ be an $(1,1)$-ultramodular overlap function. Then, the function $IF: L([0,1])^2 \rightarrow L([0,1])$, such that $IF = \widehat{O}$ is an w-iv-overlap function for the tuple $(\leq_{Pr},\leq_{Pr},O^d)$, where $O^d$ is the dual of $O$.
\end{proposition}

\begin{proof}
Immediate from Theorem \ref{theo-ultradual}.
\end{proof}

\begin{example}
Let $O_t: [0,1]^2 \rightarrow [0,1]$ be the Ot overlap, given in Table \ref{table:exoverlaps}.  Then, the function $IF: L([0,1])^2 \rightarrow L([0,1])$, such that $IF = \widehat{O_t}$ is an w-iv-overlap function for the tuple $(\leq_{Pr},\leq_{Pr},O_t^d)$, where $O_t^d$ is the dual of $O_t$.
\end{example}

\begin{proposition}
Let $O_1, O_2, O_3: [0,1]^2 \rightarrow [0,1]$ be ultramodular overlap functions, and $O_C: [0,1]^2 \rightarrow [0,1]$ be an overlap function given, for all $x,y \in [0,1]$, by $O_C(x,y) = O_3(O_1(x,y),O_2(x,y))$. Then, the function $IF: L([0,1])^2 \rightarrow L([0,1])$, such that $IF = \widehat{O_C}$ is an w-iv-overlap function for the tuple $(\leq_{Pr},\leq_{Pr},O_C^d)$, where $O_C^d$ is the dual of $O_C$.
\end{proposition}

\begin{proof}
Immediate from Theorem \ref{theo-compultra}, Proposition \ref{prop-compov} and Theorem \ref{theo-ultradual}.
\end{proof}

\begin{example}
Consider the overlap functions $O_1,O_2, O_3, O_C: [0,1]^2 \rightarrow [0,1]$ given, for all $x,y, \in [0,1]$, respectively, by $O_1(x,y) = x^{2p}\cdot y^{2p}$, $O_2(x,y) = x^{2q}\cdot y^{2q}$, $O_3(x,y) = x\cdot y$ and $O_C(x,y)=O_3(O_1(x,y),O_2(x,y))$, with $p,q \in \mathds{N}^+$. Since $O_1,O_2$ and $ O_3$ are ultramodular, it follows that the function $IF: L([0,1])^2 \rightarrow L([0,1])$, such that $IF = \widehat{O_C}$, is an w-iv-overlap function for the tuple $(\leq_{Pr},\leq_{Pr},O_C^d)$.
\end{example}

\begin{proposition}
Let $O_1, O_2: [0,1]^2 \rightarrow [0,1]$ be ultramodular overlap functions, and $O_{\alpha}: [0,1]^2 \rightarrow [0,1]$ be an overlap function given, for all $x,y,\alpha \in [0,1]$, by $O_{\alpha}(x,y) = K_{\alpha}(O_1(x,y),O_2(x,y))$. Then, the function $IF: L([0,1])^2 \rightarrow L([0,1])$ such that $IF = \widehat{O_{\alpha}}$, for all $\alpha \in [0,1]$, is an w-iv-overlap function for $(\leq_{Pr},\leq_{Pr},O_{\alpha}^d)$, where $O_{\alpha}^d$ is the dual of $O_{\alpha}$.
\end{proposition}

\begin{proof}
Immediate from Corollary \ref{coro-convexsumultra}, Proposition \ref{prop-convexsumov} and Theorem \ref{theo-ultradual}.
\end{proof}

\begin{example}
Consider the ultramodular overlap functions $O_1,O_2,O_\alpha: [0,1]^2 \rightarrow [0,1]$ given, for all $x,y,\alpha \in [0,1]$, respectively, by $O_1(x,y) = x^2y^2$, $O_2(x,y) = x^4y^4$ and $O_{\alpha}=K_{\alpha}(O_1(x,y),O_2(x,y))$. It follows that the function $IF: L([0,1])^2 \rightarrow L([0,1])$, such that $IF = \widehat{O_{\alpha}}$, for all $\alpha \in [0,1]$, is an w-iv-overlap function for the tuple $(\leq_{Pr},\leq_{Pr},O_{\alpha}^d)$.
\end{example}

%
%
%

The following definition introduces a key concept to be applied in two of the construction methods presented latter in the paper:

\begin{definition}\label{moa}
Consider a function $B: [0,1]^2 \rightarrow [0,1]$  and let $IF: L([0,1])^2 \rightarrow L([0,1])$ be an interval-valued function. Then, the function $m_{IF,B}: L([0,1])^2 \rightarrow [0,1]$, defined for all $X,Y \in L([0,1])$ by:
\begin{equation}\label{eq-moa}
m_{IF,B}(X,Y) = \min\{w(IF(X,Y)),w(IF(Y,X)),B(w(X),w(Y)),B(w(Y),w(X))\},
\end{equation}
is called the minimal width threshold for the pair $(IF,B)$. Whenever $B$ and $IF$ are both commutative, then Equation (\ref{eq-moa}) can be reduced to:
\begin{equation*}
m_{IF,B}(X,Y) = \min\{w(IF(X,Y)),B(w(X),w(Y))\}.
\end{equation*}

\end{definition}

\begin{proposition}
 Let $m_{\widehat{F},B}: L([0,1])^2 \rightarrow [0,1]$ be the minimal width threshold for the pair $(\widehat{F},B)$ with $\widehat{F}: L([0,1])^2 \rightarrow L([0,1])$ being an interval-valued function having an increasing function  $F: [0,1]^2 \rightarrow [0,1]$ as both its representatives. Whenever it holds that: i) both $X$ and $Y$ are degenerate or ii) either $X$ or $Y$ is degenerate and $B$ is a conjunctive function, then  $m_{\widehat{F},B}(X,Y) = 0$.
\end{proposition}

\begin{proof}
Consider an increasing function $F: [0,1]^2 \rightarrow [0,1]$, a conjunctive function $B: [0,1]^2 \rightarrow [0,1]$  and the minimal width threshold $m_{\widehat{F},B}: L([0,1])^2 \rightarrow [0,1]$ given by Definition \ref{moa}. Then:
\begin{description}
  \item[i)] Take $X,Y \in L([0,1])$ such that $\underline{X}=\overline{X}$ and $\underline{Y}=\overline{Y}$, that is, both $X$ and $Y$ are degenerate. Then, we have that $w(\widehat{F}(X,Y)) = F(\overline{X},\overline{Y}) - F(\underline{X},\underline{Y}) = 0$ and, similarly,  $w(\widehat{F}(Y,X)) = 0$. So, it holds that
      \begin{equation*}
      m_{\widehat{F},B}(X,Y) = \min\{0,0,B(w(X),w(Y)),B(w(Y),w(X))\} = 0;
      \end{equation*}


  \item[ii)] Take $X,Y \in L([0,1])$ such that $\underline{X}=\overline{X}$, meaning that $w(X) = 0$. Since $B$ is conjunctive, it holds that $B(w(X),w(Y)) = B(0,w(Y)) = 0$ and, analogously, $B(w(Y),w(X)) = 0$. Then, we have that
      \begin{equation*}
      m_{\widehat{F},B}(X,Y) = \min\{w(\widehat{F}(X,Y)),w(\widehat{F}(Y,X)),0,0\} = 0.
      \end{equation*}
       The same result applies when $Y$ is degenerate. \qed
\end{description}

\end{proof}

\begin{lemma}\label{lema-prkalfa}
Consider a strict overlap function $O: [0,1]^2 \rightarrow [0,1]$ and $X,Y,Z \in L([0,1])$ such that $X \leq_{Pr} Y$ and $\underline{Z}>0$. Then, one has that:
\begin{description}
\item [a)] If $\underline{X} = \underline{Y}$ and $\overline{X} < \overline{Y}$, then $K_{\alpha}(\widehat{O}(X,Z)) <  K_{\alpha}(\widehat{O}(Y,Z)),$ for all $\alpha \in (0,1]$;
\item [b)] If $\underline{X} < \underline{Y}$ and $\overline{X} = \overline{Y}$, then $K_{\alpha}(\widehat{O}(X,Z)) < K_{\alpha}(\widehat{O}(Y,Z)),$ for all $\alpha \in [0,1)$;
\item [c)] If $\underline{X} < \underline{Y}$ and $\overline{X} < \overline{Y}$, then $K_{\alpha}(\widehat{O}(X,Z)) < K_{\alpha}(\widehat{O}(Y,Z)),$ for all $\alpha \in [0,1]$.
    \end{description}
\end{lemma}
\begin{proof}
Consider a strict overlap function $O: [0,1]^2 \rightarrow [0,1]$ and $X,Y,Z \in L([0,1])$ such that $X <_{Pr} Y$. Then, we have the following cases:
\begin{description}
\item [a)] $\underline{X} = \underline{Y}$ and $\overline{X} < \overline{Y}$. As $O$ is strict and $\overline{Z}>0$, we have that $O(\underline{X}, \underline{Z}) = O(\underline{Y}, \underline{Z})$ and $O(\overline{Y}, \overline{Z}) < O(\overline{Y}, \overline{Z})$. Then, $K_{\alpha}(\widehat{O}(X,Z)) = (1 - \alpha)\cdot O(\underline{X}, \underline{Z}) + \alpha \cdot O(\overline{X}, \overline{Z}) < (1 - \alpha)\cdot O(\underline{Y}, \underline{Z}) + \alpha \cdot O(\overline{Y}, \overline{Z}) = K_{\alpha}(\widehat{O}(Y,Z))$, for all $\alpha \in (0,1]$;
\item [b)] $\underline{X} < \underline{Y}$ and $\overline{X} = \overline{Y}$. Again, as $O$ is strict and $\underline{Z}>0$, we have that  $O(\underline{X}, \underline{Z}) < O(\underline{Y}, \underline{Z})$ and $O(\overline{Y}, \overline{Z}) = O(\overline{Y}, \overline{Z})$. So, $K_{\alpha}(\widehat{O}(X,Z)) = (1 - \alpha)\cdot O(\underline{X}, \underline{Z}) + \alpha \cdot O(\overline{X}, \overline{Z}) < (1 - \alpha)\cdot O(\underline{Y}, \underline{Z}) + \alpha \cdot O(\overline{Y}, \overline{Z}) = K_{\alpha}(\widehat{O}(Y,Z))$, for all $\alpha \in [0,1)$;
\item [c)] $\underline{X} < \underline{Y}$ and $\overline{X} < \overline{Y}$. Analogously to the other cases, we have that $O(\underline{X}, \underline{Z}) < O(\underline{Y}, \underline{Z})$ and $O(\overline{Y}, \overline{Z}) < O(\overline{Y}, \overline{Z})$. Thus, $K_{\alpha}(\widehat{O}(X,Z)) = (1 - \alpha)\cdot O(\underline{X}, \underline{Z}) + \alpha \cdot O(\overline{X}, \overline{Z}) < (1 - \alpha)\cdot O(\underline{Y}, \underline{Z}) + \alpha \cdot O(\overline{Y}, \overline{Z}) = K_{\alpha}(\widehat{O}(Y,Z))$, for all $\alpha \in [0,1]$. \qed
    \end{description}
\end{proof}


Here, we present the first construction method for w-iv-overlap functions:

\begin{theorem}\label{teo-width}

Consider a commutative and increasing function $B: [0,1]^2 \rightarrow [0,1]$, a strict overlap function $O: [0,1]^2 \rightarrow [0,1]$ and take $\alpha \in (0,1]$ and $\beta \in [0,\alpha)$. Then, the interval-valued function $IOw_{B}^{\alpha}: L([0,1])^2 \rightarrow L([0,1])$ defined, for all $X,Y \in L([0,1])$, by
\begin{equation}\label{eq-iow}
IOw_{B}^{\alpha}(X,Y)= [K_{\alpha}(\widehat{O}(X,Y)) - \alpha \cdot m_{\widehat{O},B}(X,Y),  K_{\alpha}(\widehat{O}(X,Y)) + (1 - \alpha)\cdot m_{\widehat{O},B}(X,Y)],
\end{equation}
is a w-iv-overlap function for the tuple $(\leq_{Pr},\leq_{\alpha,\beta}, B)$.

\end{theorem}

\begin{proof}
See \ref{ap-proofs}.
\end{proof}

\begin{proposition}\label{prop-contidorep}
Let $O: [0,1]^2 \rightarrow [0,1]$ be a strict overlap function, $B: [0,1]^2 \rightarrow [0,1]$ be an increasing and commutative function and $IOw_{B}^{\alpha}: L([0,1])^2 \rightarrow L([0,1])$ be an w-iv-overlap function for the tuple $(\leq_{Pr},\leq_{\alpha,\beta},B)$ obtained through Theorem \ref{teo-width} for any $\alpha,\beta \in [0,1]$ such that $\alpha \neq \beta$. Then, for any $X,Y \in L([0,1])$ one has that $IOw_{B}^{\alpha}(X,Y) \subseteq \widehat{O}(X,Y)$.
\end{proposition}

\begin{proof}
It is immediate that $K_{\alpha}(IOw_{B}^{\alpha}(X,Y)) = K_{\alpha}(\widehat{O}(X,Y))$, for any $\alpha \in [0,1]$. Then, either $IOw_{B}^{\alpha}(X,Y) \subseteq \widehat{O}(X,Y)$ or $\widehat{O}(X,Y) \subseteq IOw_{B}^{\alpha}(X,Y)$. On the other hand, as
\begin{equation*}
w(IOw_{B}^{\alpha}(X,Y)) = m_{\widehat{O},B}(X,Y) = \min\{w(\widehat{O}(X,Y)),B(w(X),w(Y))\} \leq w(\widehat{O}(X,Y)),
\end{equation*}
then $IOw_{B}^{\alpha}(X,Y) \subseteq \widehat{O}(X,Y)$.
  \qed
\end{proof}

The next result is immediate from Theorem \ref{theo-ultradual}.

\begin{proposition}\label{prop-iowigualrep}
Let $O: [0,1]^2 \rightarrow [0,1]$ be an $(1,1)$-ultramodular overlap function, $A: [0,1]^2 \rightarrow [0,1]$ be an aggregation function such that $A \geq O^d$ and $IOw_{A}^{\alpha}: L([0,1])^2 \rightarrow L([0,1])$ be the w-iv-overlap function  for the tuple $(\leq_{Pr},\leq_{\alpha,\beta}, A)$, obtained by Theorem \ref{teo-width} with $\alpha,\beta \in [0,1]$. Then, $IOw_{A}^{\alpha}(X,Y) = \widehat{O}(X,Y)$, for all $X,Y \in L([0,1])$.
\end{proposition}

\begin{remark}

From Proposition \ref{prop-iowigualrep}, when we apply construction method presented in Theorem \ref{teo-width} to obtain an w-iv-overlap function $IOw_{A}^{\alpha}$ based on an $(1,1)$-ultramodular overlap function $O$ with a width-limiting aggregation function $A$, such that $A < O^d$ and $\alpha,\beta \in [0,1]$, the output interval is narrower (with greater quality of information) than the one obtained by $\widehat{O}$. Furthermore, from Proposition \ref{prop-contidorep}, it holds that this interval is contained in the one obtained by $\widehat{O}$, which is a desirable property, since $\widehat{O}$ is the best interval representation of $O$, in the sense of \cite{BEDREGAL20101373, Dimuro2008123}.

\end{remark}

The following examples aim to illustrate how the construction method presented in Theorem \ref{teo-width} works, comparing the results with the ones obtained through $o$-representable iv-overlap functions.

\begin{example}\label{ex-con-w}
Consider an increasing and commutative function $B: [0,1]^2 \rightarrow [0,1]$, the product overlap function $Op: [0,1]^2 \rightarrow [0,1]$, $\alpha \in (0,1]$ and $\beta \in [0,\alpha)$. Then, the interval-valued function $IOpw_{B}^{\alpha}: L([0,1])^2 \rightarrow L([0,1])$ defined, for all $X,Y \in L([0,1])$, by
\begin{equation}\label{eq-ex1}
IOpw_{B}^{\alpha}(X,Y)= [K_{\alpha}(\widehat{Op}(X,Y)) - \alpha \cdot m_{\widehat{Op},B}(X,Y),  K_{\alpha}(\widehat{Op}(X,Y)) + (1 - \alpha)\cdot m_{\widehat{Op},B}(X,Y)],
\end{equation}
is a w-iv-overlap function for the tuple $(\leq_{Pr},\leq_{\alpha,\beta}, \max)$.

\begin{description}
\item[1)] Take $B = \max$, $X = [0.2, 0.8]$ and $Y = [0.5, 1]$. So, we have that $\widehat{Op}([0.2, 0.8],[0.5, 1]) = [0.1,0.8]$. It is clear that $\widehat{Op}$ is not width-limited by $\max$, as $w(\widehat{Op}([0.2, 0.8],[0.5, 1])) = 0.7 > 0.6 = \max(w([0.2, 0.8]),w([0.5, 1]))$. Also, by Equation (\ref{int-kalfa-w}), observe that $\widehat{Op}([0.2, 0.8],[0.5, 1])$ can be obtained as:
\begin{equation}\label{eq-rep1}
\widehat{Op}([0.2, 0.8],[0.5, 1]) = [K_{\alpha}([0.1,0.8]) - \alpha \cdot 0.7,  K_{\alpha}([0.1,0.8]) + (1 - \alpha)\cdot 0.7],
\end{equation}
 which also results in $[0.1,0.8]$, for all $\alpha \in (0,1]$.

The minimal width threshold for the pair $(Op,\max)$ in this context is given by
\begin{eqnarray*}
\lefteqn{m_{\widehat{Op},\max}([0.2, 0.8],[0.5, 1]) =}&&\\
&&\min\{w(\widehat{Op}([0.2, 0.8],[0.5, 1])),\max(w([0.2, 0.8]),w([0.5, 1]))\} = \min\{0.7, \max\{0.6,0.5\}\} = 0.6.
\end{eqnarray*}

By Equation (\ref{eq-ex1}), we have that
\begin{equation}\label{eq-iopw1}
IOpw_{\max}^{\alpha}([0.2, 0.8],[0.5, 1]) = [K_{\alpha}([0.1,0.8]) - \alpha \cdot 0.6,  K_{\alpha}([0.1,0.8]) + (1 - \alpha)\cdot 0.6],
\end{equation}

and $w(IOpw_{\max}^{\alpha}([0.2, 0.8],[0.5, 1])) = 0.6 \leq \max(w([0.2, 0.8]),w([0.5, 1]))$, which is expected as $IOpw_{\max}^{\alpha}$ is width-limited by $\max$.

Notice, from Equations (\ref{eq-rep1}) and (\ref{eq-iopw1}), that  $K_{\alpha}(\widehat{Op}([0.2, 0.8],[0.5, 1])) = K_{\alpha}(IOpw_{\max}^{\alpha}([0.2, 0.8],[0.5, 1]) )$, and that $w(\widehat{Op}([0.2, 0.8],[0.5, 1])) = 0.7 > 0.6 = w(IOpw_{\max}^{\alpha}([0.2, 0.8],[0.5, 1]) )$.

Let us assign some values for $\alpha$ to observe what is the resulting interval for $IOpw_{\max}^{\alpha}([0.2, 0.8],[0.5, 1])$.

\begin{description}
\item [a)] If $\alpha = 0.01$, then
\begin{equation*}
IOp_{\max}^{0.01}([0.2, 0.8],[0.5, 1]) = [K_{0.01}([0.1,0.8]),  K_{0.01}([0.1,0.8]) +  0.6] = [0.107,0.707];
\end{equation*}
\item [b)] If $\alpha = 0.5$, then
\begin{equation*}
IOp_{\max}^{0.5}([0.2, 0.8],[0.5, 1]) = [K_{0.5}([0.1,0.8]) - 0.5 \cdot 0.6,  K_{0.5}([0.1,0.8]) +  0.5 \cdot 0.6] = [0.15,0.75];
\end{equation*}
\item [c)] If $\alpha = 1$, then
\begin{equation*}
 IOp_{\max}^{1}([0.2, 0.8],[0.5, 1]) = [K_{1}([0.1,0.8]) - 0.6,  K_{1}([0.1,0.8])] = [0.2,0.8].
\end{equation*}

\end{description}

\item [2)] Now, consider $B = \max$ and take $X = [0.6, 0.9]$ and $Y = [0.8,0.8]$.  So, we have that
\begin{equation*}
\widehat{Op}([0.6, 0.9],[0.8, 0.8]) = [0.48,0.72].
\end{equation*}
Although  $\widehat{Op}$ is not width-limited by $\max$, in this case it holds that as
\begin{equation*}
 w(\widehat{Op}([0.6, 0.9],[0.8, 0.8])) = 0.24 < 0.3 = \max(w([0.6, 0.9]),w([0.8, 0.8])).
\end{equation*}
Moreover,  by Equation (\ref{int-kalfa-w}), $\widehat{Op}([0.6, 0.9],[0.8, 0.8])$ can be written as:
\begin{equation*}
\widehat{Op}([0.6, 0.9],[0.8, 0.8]) = [K_{\alpha}([0.48,0.72]) - \alpha \cdot 0.24,  K_{\alpha}([0.48,0.72]) + (1 - \alpha)\cdot 0.24] = [0.48,0.72].
\end{equation*}

The minimal width threshold for the pair $(Op,\max)$ in this context is given by
\begin{eqnarray*}
\lefteqn{m_{\widehat{Op},\max}([0.6, 0.9],[0.8, 0.8]) =}&&\\
&&\min\{w(\widehat{Op}([0.6, 0.9],[0.8, 0.8])),\max(w([0.6, 0.9]),w([0.8, 0.8]))\} = \min\{0.24, \max\{0.3,0\}\} = 0.24.
\end{eqnarray*}

By Equation (\ref{eq-ex1}), we have that
\begin{equation*}
IOpw_{\max}^{\alpha}([0.6, 0.9],[0.8, 0.8]) = [K_{\alpha}([0.48,0.72]) - \alpha \cdot 0.24,  K_{\alpha}([0.48,0.72]) + (1 - \alpha)\cdot 0.24] = [0.48,0.72].
\end{equation*}

Thus, $\widehat{Op}([0.6, 0.9],[0.8, 0.8]) = IOpw_{\max}^{\alpha}([0.6, 0.9],[0.8, 0.8]) = [0.48,0.72]$, for all $\alpha \in (0,1]$.

\item [3)]Next, take the same $X = [0.6, 0.9]$, and $Y = [0.8,0.8]$, but now with $B = \min$. So,
\begin{eqnarray*}
\lefteqn{m_{\widehat{Op},\min}([0.6, 0.9],[0.8,0.8]) =}&&\\
&&\min\{w(\widehat{Op}([0.6, 0.9],[0.8,0.8])), \min\{w([0.6, 0.9]),w([0.8,0.8])\}\} = \min\{0.24, \min\{0.3,0\}\} = 0.
\end{eqnarray*}
and, therefore,
\begin{equation*}
IOp_{\min}^{\alpha}([0.6, 0.9],[0.8,0.8])= [K_{\alpha}([0.48,0.72]),  K_{\alpha}([0.48,0.72])],
\end{equation*}
for any $\alpha \in (0,1]$. One can observe that $w(IOp_{\min}^{\alpha}([0.6, 0.9],[0.8,0.8]))=0$,  which is expected from Remark \ref{rem-def} as $Y = [0.8,0.8]$ is degenerate and $\min$ is a conjunctive function.

\item [4)] Finally, take $X = [0.2, 0.8]$ and $Y = [0.5, 1]$, and let $B = Op^d $. Then, the minimal width threshold for the pair $(\widehat{Op},Op^d)$ is given by
    \begin{eqnarray*}
\lefteqn{m_{\widehat{Op},Op^d}([0.2, 0.8],[0.5, 1]) =}&&\\
&&\min\{w(\widehat{Op}([0.2, 0.8],[0.5, 1])),Op^d(w([0.2, 0.8]),w([0.5, 1]))\} = \min\{0.7, Op^d(0.6,0.5)\} = 0.7.
\end{eqnarray*}
By Equation (\ref{eq-ex1}), we have that
\begin{eqnarray*}\label{eq-iopw4}
\lefteqn{IOpw_{Op^d}^{\alpha}([0.2, 0.8],[0.5, 1])}&&\\
& = &[K_{\alpha}([0.1,0.8]) - \alpha \cdot 0.7,  K_{\alpha}([0.1,0.8]) + (1 - \alpha)\cdot 0.7] \\
& = & \widehat{Op}([0.2, 0.8],[0.5, 1]) \\
& = &[0.1,0.8],
\end{eqnarray*}
which is expected, by Proposition \ref{prop-iowigualrep}, since $Op$ is an $(1,1)$-ultramodular overlap function.

\end{description}
\end{example}

Next, we present the second construction method for w-iv-overlap functions:

\begin{theorem}\label{theo-benja-geral}Let $O: [0,1]^2 \rightarrow [0,1]$ be a strict overlap function, $B: [0,1]^2 \rightarrow [0,1]$ be a commutative, increasing and conjunctive function and $\alpha\in (0,1)$, $\beta \in [0,1]$ such that $\alpha\neq\beta$. Then $IOw_{B}^{\alpha}: L([0,1])^2\rightarrow L([0,1])$ defined, for all $X,Y \in L([0,1])$, by
\begin{eqnarray*}
IOw_{B}^{\alpha}(X,Y)=[O(K_{\alpha}(X),K_{\alpha}(Y))-\alpha \theta, O(K_{\alpha}(X), K_{\alpha}(Y))+(1-\alpha)\theta],
\end{eqnarray*}
 where
\begin{eqnarray*}
\theta=B(B(w(X), w(Y)), B(O(K_{\alpha}(X), K_{\alpha}(Y)),1 - O(K_{\alpha}(X), K_{\alpha}(Y))))
\end{eqnarray*}
 is a w-iv-overlap function for the tuple $(\leq_{\alpha,\beta},\leq_{\alpha,\beta},B)$.
\end{theorem}

\begin{proof}
See \ref{ap-proofs2}.
\end{proof}

The following result is immediate as a w-iv-overlap function for the tuple $(\leq_{\alpha,\beta},\leq_{\alpha,\beta},B)$ is also a $\leq_{\alpha,\beta}$-overlap function (Definition \ref{def-int-AD-n-overlap}), in the sense of \cite{ASMUS2020-TFS}.

\begin{corollary}
\label{coro-benja-geral}Let $O: [0,1]^2 \rightarrow [0,1]$ be a strict overlap function, $B: [0,1]^2 \rightarrow [0,1]$ be a commutative, increasing and conjunctive function and $\alpha\in (0,1)$, $\beta \in [0,1]$ such that $\alpha\neq\beta$. Then $IOw_{B}^{\alpha}: L([0,1])^2\rightarrow L([0,1])$ defined, for all $X,Y \in L([0,1])$, by
\begin{eqnarray*}
IOw_{B}^{\alpha}(X,Y)=[O(K_{\alpha}(X),K_{\alpha}(Y))-\alpha \theta, O(K_{\alpha}(X), K_{\alpha}(Y))+(1-\alpha)\theta],
\end{eqnarray*}
 where
\begin{eqnarray*}
\theta=B(B(w(X), w(Y)), B(O(K_{\alpha}(X), K_{\alpha}(Y)),1 - O(K_{\alpha}(X), K_{\alpha}(Y))))
\end{eqnarray*}
 is a $\leq_{\alpha,\beta}$-overlap function.
\end{corollary}

\begin{example}\label{ex-con-benja}
Consider a function $B: [0,1]^2 \rightarrow [0,1]$ such that $B = \min$ and the product overlap function $Op: [0,1]^2 \rightarrow [0,1]$. Then, the interval-valued function $IOpw_{\min}^{\alpha}: L([0,1])^2 \rightarrow L([0,1])$ defined, for all $X,Y \in L([0,1])$, by
\begin{eqnarray}\label{eq-benja}
IOpw_{\min}^{\alpha}(X,Y)=[Op(K_{\alpha}(X),K_{\alpha}(Y))-\alpha \theta, Op(K_{\alpha}(X), K_{\alpha}(Y))+(1-\alpha)\theta],
\end{eqnarray}
 where
\begin{eqnarray*}
\theta=\min(\min(w(X), w(Y)), \min(Op(K_{\alpha}(X), K_{\alpha}(Y)),1 - Op(K_{\alpha}(X), K_{\alpha}(Y))))
\end{eqnarray*}
is a w-iv-overlap function for the tuple $(\leq_{\alpha,\beta},\leq_{\alpha,\beta}, \min)$, for all $\alpha \in (0,1), \beta \in [0,1]$ with $\alpha \neq \beta$.

\begin{description}
\item[1)] Take $X = [0.2, 0.8]$, and $Y = [0.5, 1]$. By Equation (\ref{eq-benja}), we have that
\begin{eqnarray*}
\lefteqn{IOpw_{\min}^{\alpha}([0.2, 0.8],[0.5, 1])}&&\\
&=&[Op(K_{\alpha}([0.2, 0.8]),K_{\alpha}([0.5, 1]))-\alpha \theta, Op(K_{\alpha}([0.2, 0.8]), K_{\alpha}([0.5, 1]))+(1-\alpha)\theta],
\end{eqnarray*}
 where
\begin{eqnarray*}
\lefteqn{\theta=\min(\min(w([0.2, 0.8]), w([0.5, 1])),}&&\\
&& \min(Op(K_{\alpha}([0.2, 0.8]), K_{\alpha}([0.5, 1])),1 - Op(K_{\alpha}([0.2, 0.8]), K_{\alpha}([0.5, 1]))))
\end{eqnarray*}

Let us assign some values for $\alpha$ to observe what is the resulting interval for $IOpw_{\min}^{\alpha}([0.2, 0.8],[0.5, 1])$.

\begin{description}
\item [a)] If $\alpha = 0.01$ then
\begin{eqnarray*}
\lefteqn{ \theta=\min(\min(w([0.2, 0.8]), w([0.5, 1])),}&&\\
&& \min(Op(K_{0.01}([0.2, 0.8]), K_{0.01}([0.5, 1])),1 - Op(K_{0.01}([0.2, 0.8]), K_{0.01}([0.5, 1]))))\\
&=& \min(\min(0.6,0.5),\min(0.104,0.896)) = 0.104
\end{eqnarray*}
and
\begin{equation*}
IOpw_{\min}^{0}([0.2, 0.8],[0.5, 1]) = [0.104 - 0.01 \cdot 0.104, 0.104 +  0.99 \cdot 0.104 ] = [0.103,0.207];
\end{equation*}

\item [b)] If $\alpha = 0.5$ then
\begin{eqnarray*}
\lefteqn{ \theta=\min(\min(w([0.2, 0.8]), w([0.5, 1])),}&&\\
&& \min(Op(K_{0.5}([0.2, 0.8]),K_{0.5}([0.5, 1])),1 - Op(K_{0.5}([0.2, 0.8]), K_{0.5}([0.5, 1]))))\\
&  = & \min(\min(0.6,0.5),\min(0.375,0.625)) = 0.375
\end{eqnarray*}
 and
 \begin{equation*}
 IOpw_{\min}^{0.5}([0.2, 0.8],[0.5, 1]) = [0.375 - 0.5 \cdot 0.375, 0.375 + 0.5 \cdot 0.375] = [0.1875,0.5625];
 \end{equation*}

\item [c)] If $\alpha = 0.99$ then
\begin{eqnarray*}
\lefteqn{\theta=\min(\min(w([0.2, 0.8]), w([0.5, 1])),}&&\\
&&\min(Op(K_{0.99}([0.2, 0.8]), K_{0.99}([0.5, 1])),1 - Op(K_{0.99}([0.2, 0.8]), K_{0.99}([0.5, 1]))))\\
& =& \min(\min(0.6,0.5),\min(0.79,0.2099)) = 0.2099
\end{eqnarray*}
and
\begin{equation*}
IOpw_{\min}^{0.99}([0.2, 0.8],[0.5, 1]) = [0.79 - 0.99*0.2099, 0.79 + 0.01*2099] = [0.5822,0.7921].
\end{equation*}
\end{description}

\item [2)] Now, take $X = [0.6, 0.9]$, and $Y = [0.8,0.8]$. By Equation (\ref{eq-benja}), we have that
\begin{eqnarray*}
\lefteqn{IOpw_{\min}^{\alpha}([0.6, 0.9],[0.8,0.8])}&&\\
&&=[Op(K_{\alpha}([0.6, 0.9]),K_{\alpha}([0.8,0.8]))-\alpha \theta, Op(K_{\alpha}([0.6, 0.9]), K_{\alpha}([0.8,0.8]))+(1-\alpha)\theta],
\end{eqnarray*}
 where
\begin{eqnarray*}
&&\theta=\min(\min(w([0.6, 0.9],w([0.8,0.8])), \min(Op(K_{\alpha}([0.6, 0.9]),\\
&&K_{\alpha}([0.8,0.8])),1 - Op(K_{\alpha}([0.6, 0.9]), K_{\alpha}([0.8, 0.8])))) = 0.
\end{eqnarray*}

Thus, $IOpw_{\min}^{\alpha}([0.6, 0.9],[0.8,0.8])= [K_{\alpha}([0.6, 0.9]) \cdot K_{\alpha}([0.8,0.8])]$, for any $\alpha \in (0,1)$. For example:

\begin{description}
\item [a)] If $\alpha = 0.01$ then $IOpw_{\min}^{0.01}([0.6, 0.9],[0.8,0.8]) = [0.603 \cdot 0.8, 0.603 \cdot 0.8] = [0.4824,0.4824]$;

\item [a)] If $\alpha = 0.5$ then $IOpw_{\min}^{0.5}([0.6, 0.9],[0.8,0.8]) = [0.7 \cdot 0.8, 0.7 \cdot 0.8] = [0.56,0.56]$;
\item [a)] If $\alpha = 0.99$ then $IOpw_{\min}^{0.99}([0.6, 0.9],[0.8,0.8]) =  [0.897 \cdot 0.8, 0.897 \cdot 0.8] = [0.7176,0.7176]$.
\end{description}
\end{description}
\end{example}

\begin{remark}
Considering Theorem \ref{theo-benja-geral}, when $B=\min$ we recover the construction method presented in Theorem \ref{theo-benja}, meaning that Theorem \ref{theo-benja-geral} is more general. Also, it is noteworthy that the reason for $\alpha \in (0,1)$ is to assure that the construction method produces an w-iv-overlap function. For example, if $\alpha = 0$, then $IOw_{B}^{0}([0,1],[0.2,0.2]) = [0,0]$, which would contradict \textbf{(IOw2)}. Also, one can observe that $IOw_{B}^{\alpha}$ falls into the conditions of Remark \ref{rem-def}, meaning that if either $X$ or $Y$ is degenerate, then $IOw_{B}^{\alpha}(X,Y)$ is also degenerate, as shown in Example \ref{ex-con-benja}, for $X = [0.6, 0.9]$ and $Y =[0.8,0.8]$. Finally, although the w-iv-overlap constructed by the method presented in Theorem \ref{theo-benja-geral} is width-limited by the chosen function $B$, the output interval may not be contained in the best interval representation of the chosen overlap function $O$, as shown in the next example.
\end{remark}

\begin{example}
Consider an w-iv-overlap function $IOpw_{\min}^{0.99}$ for the tuple $(\leq_{0.99,\beta},\leq_{0.99,\beta}, \min)$ obtained via the construction method presented in Theorem \ref{theo-benja-geral} by taking $B = \min$, $O = O_P$ (the product overlap) and $\beta \in [0,1]$ such that $\beta \neq 0.99$. In the case when $X = Y = [0.1,0.4]$, we have that
\begin{equation*}
\widehat{Op}([0.1,0.4],[0.1,0.4]) = [0.1 \cdot 0.1, 0.4 \cdot 0.4] = [0.01,0.16].
\end{equation*}
From Theorem \ref{theo-benja-geral}, it holds that
\begin{eqnarray*}
\lefteqn{\theta=\min(\min(w([0.1,0.4]), w([0.1,0.4])),}&&\\
 && \min(Op(K_{0.99}([0.1,0.4]), K_{0.99}([0.1,0.4])),1 - Op(K_{0.99}([0.1,0.4]), K_{0.99}([0.1,0.4]))))\\
&=&\min(\min(0.3, 0.3), \min(Op(0.397, 0.397),1 - Op(0.397,0.397))))\\
&=&\min(0.3, \min(0.1576,0.8424)))\\
&=& 0.1576.
\end{eqnarray*}
So,
\begin{eqnarray*}
\lefteqn{IOpw_{\min}^{0.99}([0.1,0.4],[0.1,0.4]) = [Op(K_{0.99}([0.1,0.4]), K_{0.99}([0.1,0.4])) - 0.99 \cdot 0.1576,}&&\\
 && Op(K_{0.99}([0.1,0.4]), K_{0.99}([0.1,0.4])) + 0.01 \cdot 0.1576] \\
& = & [0.0016, 0.1502],
\end{eqnarray*}
showing that $IOpw_{\min}^{0.99}([0.1,0.4],[0.1,0.4]) \nsubseteq \widehat{Op}([0.1,0.4],[0.1,0.4])$.
\end{example}

Before presenting the third construction method for w-iv-overlaps, let us recall some important concepts presented in \cite{BUSTINCE202023}:

\begin{definition}\label{def-lamda}
Let $c \in [0,1]$ and $\alpha \in [0,1]$. We denote by $d_{\alpha}(c)$ the maximal possible width of an interval $Z \in L([0,1])$ such that $K_{\alpha}(Z) = c$. Moreover, for any $X \in L([0,1])$, define
\begin{equation*}
\lambda_{\alpha}(X)=\frac{w(X)}{d_{\alpha}(K_{\alpha}(X))},
\end{equation*}
where we set $\frac{0}{0}=1$.
\end{definition}

\begin{proposition}\label{def-dalpha}
For all $\alpha \in [0,1]$ and $X \in L([0,1])$ it holds that
\begin{equation*}
d_{\alpha}(K_{\alpha}(X)) = \min\left\{\frac{K_{\alpha}(X)}{\alpha},\frac{1 - K_{\alpha}(X)}{1 - \alpha}\right\},
\end{equation*}
where we set $\frac{r}{0}=1$, for all $r \in [0,1]$.
\end{proposition}

Now, we present a version of Theorem 3.16 in \cite{BUSTINCE202023} in the context of $2$-dimensional functions.
\begin{theorem}
Let $\alpha,\beta \in [0,1]$, such that, $\alpha \neq \beta$. Let $A_1, A_2: [0,1]^2 \rightarrow [0,1]$ be two aggregation functions where $A_1$ is strictly increasing. Then $IF^{\alpha}:L([0, 1])^2 \rightarrow L([0, 1])$ defined by:
\begin{eqnarray*}
IF_{A1,A2}^{\alpha}(X,Y)= R, \, \,  \mbox{where},\left\{
\begin{array}{l}
{K_{\alpha}(R)=A_1(K_{\alpha}(X),K_{\alpha}(Y))}, \\
{\lambda_{\alpha}(R)=A_2(\lambda_{\alpha}(X),\lambda_{\alpha}(Y))},
\end{array}
\right.
\end{eqnarray*}
for all $X,Y \in L([0,1])$, is an $\leq_{\alpha,\beta}$-increasing iv-aggregation function.
\end{theorem}

\begin{proof}
It follows from Theorem 3.16 in \cite{BUSTINCE202023}.
\qed
\end{proof}

As overlap functions are a class of aggregation functions, the following result is immediate.

\begin{corollary}\label{coro-zdenko}
Let $\alpha,\beta \in [0,1]$, such that, $\alpha \neq \beta$. Let $O: [0,1]^2 \rightarrow [0,1]$ be a strict overlap function and $A: [0,1]^2 \rightarrow [0,1]$ be an aggregation function. Then $IF_{O,A}^{\alpha}:L([0, 1])^2 \rightarrow L([0, 1])$ defined by:
\begin{eqnarray*}
IF_{O,A}^{\alpha}(X,Y)= R, \, \,  \mbox{where},\left\{
\begin{array}{l}
{K_{\alpha}(R)=O(K_{\alpha}(X),K_{\alpha}(Y))}, \\
{\lambda_{\alpha}(R)=A(\lambda_{\alpha}(X),\lambda_{\alpha}(Y))},
\end{array}
\right.
\end{eqnarray*}
for all $X,Y \in L([0,1])$, is an $\leq_{\alpha,\beta}$-increasing iv-aggregation function.
\end{corollary}

The following result is immediate from Definition \ref{def-lamda} and Corollary \ref{coro-zdenko}.
\begin{corollary}\label{coro-wzdenko}
Let $\alpha,\beta \in [0,1]$ be such that, $\alpha \neq \beta$. Let $O: [0,1]^2 \rightarrow [0,1]$ be a strict overlap function, $A: [0,1]^2 \rightarrow [0,1]$ be an aggregation function and $IF_{O,A}^{\alpha}:L([0, 1])^2 \rightarrow L([0, 1])$ be an iv-aggregation function constructed as in Corollary \ref{coro-zdenko}. Then, for all $X,Y \in L([0,1])$, we have that
\begin{equation*}
w(IF_{O,A}^{\alpha}(X,Y)) = A(\lambda_{\alpha}(X),\lambda_{\alpha}(Y)) \cdot d_{\alpha}(K_{\alpha}(IF_{O,A}^{\alpha}(X,Y))).
\end{equation*}
\end{corollary}

Finally, the third construction method for w-iv-overlaps is obtained as follows:

\begin{theorem}\label{theo-conzdenko}
Consider a strict overlap function $O: [0,1]^2 \rightarrow [0,1]$, a commutative aggregation function $B: [0,1]^2 \rightarrow [0,1]$, an interval-valued aggregation function $IF_{O,B}^{\alpha}:L([0, 1])^2 \rightarrow L([0, 1])$ defined as in Corollary \ref{coro-zdenko}, the minimal width threshold $m_{IF_{O,B}^{\alpha},B}: L([0, 1])^2 \rightarrow L([0, 1])$ for the pair $(IF_{O,B}^{\alpha},B)$, $\alpha \in (0,1)$ and $\beta \in [0,1]$ with $\alpha \neq \beta$. Then, the interval-valued function $IOw_{B}^{\alpha}: L([0,1])^2 \rightarrow L([0,1])$ defined by
\begin{equation*}
IOw_{B}^{\alpha}(X,Y)= R,
\end{equation*}
where:
\begin{description}
  \item[(i)] $K_{\alpha}(R)=O(K_{\alpha}(X),K_{\alpha}(Y))$;
  \item[(ii)] $w(R) = m_{IF_{O,B}^{\alpha},B}(X,Y)$.
\end{description}
is a w-iv-overlap function for the tuple $(\leq_{\alpha,\beta},\leq_{\alpha,\beta},B)$.

\end{theorem}
	
\begin{proof}
See \ref{ap-proofs3}.
\end{proof}

The following result is immediate as a w-iv-overlap function for the tuple $(\leq_{\alpha,\beta},\leq_{\alpha,\beta},B)$ is also a $\leq_{\alpha,\beta}$-overlap function (Definition \ref{def-int-AD-n-overlap}), in the sense of \cite{ASMUS2020-TFS}.

\begin{corollary}\label{coro-conzdenko}
Consider a strict overlap function $O: [0,1]^2 \rightarrow [0,1]$, a commutative aggregation function $B: [0,1]^2 \rightarrow [0,1]$, an interval-valued aggregation function $IF_{O,B}^{\alpha}:L([0, 1])^2 \rightarrow L([0, 1])$ defined as in Corollary \ref{coro-zdenko}, the minimal width threshold $m_{IF_{O,B}^{\alpha},B}: L([0, 1])^2 \rightarrow L([0, 1])$ for the pair $(IF_{O,B}^{\alpha},B)$, $\alpha \in (0,1)$ and $\beta \in [0,1]$ with $\alpha \neq \beta$. Then, the interval-valued function $IOw_{B}^{\alpha}: L([0,1])^2 \rightarrow L([0,1])$ defined by
\begin{equation*}
IOw_{B}^{\alpha}(X,Y)= R,
\end{equation*}
where:
\begin{description}
  \item[(i)] $K_{\alpha}(R)=O(K_{\alpha}(X),K_{\alpha}(Y))$;
  \item[(ii)] $w(R) = m_{IF_{O,B}^{\alpha},B}(X,Y)$.
\end{description}
is a $\leq_{\alpha,\beta}$-overlap function.

\end{corollary}

\begin{example}\label{ex-con-zdenko}
Consider a commutative aggregation function $B: [0,1]^2 \rightarrow [0,1]$ and the product overlap function $Op: [0,1]^2 \rightarrow [0,1]$. Then, the interval-valued function $IOpw_{B}^{\alpha}: L([0,1])^2 \rightarrow L([0,1])$ defined, for all $X,Y \in L([0,1])$, by
\begin{equation*}
IOpw_{B}^{\alpha}(X,Y)= R,
\end{equation*}
where:
\begin{description}
  \item[(i)] $K_{\alpha}(R)=Op(K_{\alpha}(X),K_{\alpha}(Y))$;
  \item[(ii)] $w(R) = m_{IF_{Op,B}^{\alpha},B}(X,Y)$.
\end{description}
is a w-iv-overlap function for the tuple $(\leq_{\alpha,\beta},\leq_{\alpha,\beta}, B)$, for all $\alpha,\beta \in [0,1]$ such that $\alpha \neq \beta$.

\begin{description}
\item[1)] Take $B=\max$, $X = [0.2, 0.8]$ and $Y = [0.5, 1]$. By \textbf{(i)}, we have that
\begin{equation*}
K_{\alpha}(R)=Op(K_{\alpha}([0.2, 0.8]),K_{\alpha}([0.5, 1]))
\end{equation*}
and
\begin{eqnarray*}
\lefteqn{w(R) = m_{IF_{Op,\max}^{\alpha},\max}([0.2, 0.8],[0.5, 1]) \, \, \mbox{by \textbf{(ii)}}}&&\\
& =& \min\{w(IF_{O,\max}^{\alpha}([0.2, 0.8],[0.5, 1])), \max(w([0.2, 0.8]),w([0.5, 1]))\} \, \, \mbox{by Definition \ref{moa}}\\
& =& \min\{\max(\lambda_{\alpha}([0.2, 0.8]),\lambda_{\alpha}([0.5, 1])) \cdot d_{\alpha}(K_{\alpha}(IF_{O,\max}^{\alpha}([0.2, 0.8],[0.5, 1]))), \max(0.6,0.5)\}\\
&&\hspace{9.2cm} \mbox{by Corollary \ref{coro-wzdenko}}\\
& =& \min\{\max(\lambda_{\alpha}([0.2, 0.8]),\lambda_{\alpha}([0.5, 1])) \cdot d_{\alpha}(Op(K_{\alpha}([0.2, 0.8]),K_{\alpha}([0.5, 1]))), 0.6\}.\\
&& \hspace{9.2cm} \mbox{by Corollary \ref{coro-zdenko}}
\end{eqnarray*}

Let us assign some values for $\alpha$ to observe what is the resulting interval for $IOpw_{\max}^{\alpha}([0.2, 0.8],[0.5, 1])$.

\begin{description}
\item [a)] If $\alpha = 0.01$ then
\begin{equation*}
K_{0.01}(R)=Op(K_{0.01}([0.2, 0.8]),K_{0.01}([0.5, 1])) = 0.206 \cdot 0.505 = 0.104,
\end{equation*}
and
\begin{eqnarray*}
\lefteqn{ w(R)}&&\\
&=& \min\{\max(\lambda_{0.01}([0.2, 0.8]),\lambda_{0.01}([0.5, 1])) \cdot d_{0.01}(Op(K_{\alpha}([0.2, 0.8]),K_{0.01}([0.5, 1]))), 0.6\}\\
& =& \min\left\{\max\left(\frac{w([0.2, 0.8])}{d_{0.01}(K_{0.01}([0.2,0.8]))},\frac{w([0.5, 1])}{d_{0.01}(K_{0.01}([0.5,1]))}\right) \cdot d_{0.01}(0.104), 0.6\right\}\\
&& \hspace{10cm} \mbox{by Definition \ref{def-dalpha}}\\
& =& \min\left\{\max\left(\frac{0.6}{\min\left\{\frac{0.206}{0.01},\frac{0.794}{0.99}\right\}},\frac{0.5}{\min\left\{\frac{0.505}{0.01},\frac{0.495}{0.99}\right\}}\right) \cdot \min\left\{\frac{0.104}{0.01},\frac{0.896}{0.99}\right\}, 0.6\right\} \\
&& \hspace{10cm} \mbox{by Proposition \ref{def-lamda}}\\
& =& \min\left\{\max\left(\frac{0.6}{0.802},\frac{0.5}{0.5}\right) \cdot 0.905, 0.6\right\} = \min\{0.905,0.6\} = 0.6.
\end{eqnarray*}

So, by Equation (\ref{int-kalfa-w}),  $IOpw_{\max}^{0.01}([0.2,0.8],[0.5,1]) = [0.104 - 0.01 \cdot 0.6, 0.104 + 0.99 \cdot 0.6] = [0.098,0.698]$.

In the next cases, we will just present the final results.

\item [b)] If $\alpha = 0.5$ then
\begin{equation*}
K_{0.5}(R)=Op(K_{0.5}([0.2, 0.8]),K_{0.5}([0.5, 1])) = 0.5 \cdot 0.75 = 0.375,
\end{equation*}
 and
\begin{eqnarray*}
w(R)= \min\left\{\max\left(\frac{0.6}{1},\frac{0.5}{0.5}\right) \cdot 0.625, 0.6\right\} = \min\{0.625,0.6\} = 0.6.
\end{eqnarray*}

Thus, $IOpw_{\max}^{0.5}([0.2,0.8],[0.5,1]) = [0.375 - 0.5 \cdot 0.6,0.375 + 0.5 \cdot 0.6] = [0.075,0.675]$.

\item [c)] If $\alpha = 0.99$ then
\begin{equation*}
K_{0.99}(R)=Op(K_{0.99}([0.2, 0.8]),K_{0.99}([0.5, 1])) = 0.794 \cdot 0.995= 0.79,
\end{equation*}
 and
\begin{eqnarray*}
w(R)= \min\left\{\max\left(\frac{0.6}{0.802},\frac{0.5}{0.5}\right) \cdot 0.798, 0.6\right\} = \min\{0.798,0.6\} = 0.6.
\end{eqnarray*}

Therefore, $IOpw_{\max}^{0.99}([0.2,0.8],[0.5,1]) = [0.79 - 0.99 \cdot 0.6, 0.79 + 0.01 \cdot 0.6] = [0.196,0.796]$.
\end{description}

\item [2)] Now, take $X = [0.6, 0.9]$, and $Y = [0.8,0.8]$. Then, we have that
\begin{equation*}
K_{\alpha}(R)=Op(K_{\alpha}([0.6, 0.9]),K_{\alpha}([0.8, 0.8]))
\end{equation*}
and, by \textbf{(ii)},
\begin{eqnarray*}
&&w(R) = \min\{\max(\lambda_{\alpha}([0.6, 0.9]),\lambda_{\alpha}([0.8, 0.8])) \cdot d_{\alpha}(Op(K_{\alpha}([0.6, 0.9]),K_{\alpha}([0.8, 0.8]))), 0.3\}.\\
&&\hspace{10cm} \mbox{by Corollary \ref{coro-zdenko}}
\end{eqnarray*}

Once again, let us observe the value of $IOpw_{\max}^{\alpha}([0.6, 0.9],[0.8, 0.8])$ by varying the value of $\alpha$:
\begin{description}
\item [a)] If $\alpha = 0.01$ then
\begin{equation*}
K_{0.01}(R)=Op(K_{0.01}([0.6, 0.9]),K_{0.01}([0.8, 0.8])) = 0.603 \cdot 0.8 = 0.4824,
\end{equation*}
 and
\begin{eqnarray*}
w(R)= \min\left\{\max\left(\frac{0.3}{0.401},\frac{0}{0.202}\right) \cdot 0.5228, 0.3\right\} = \min\{0.3911,0.3\} = 0.3.
\end{eqnarray*}

So, $IOpw_{\max}^{0.01}([0.6,0.9],[0.8,0.8]) = [0.4824 - 0.01 \cdot 0.3, 0.4824 + 0.99 \cdot 0.3] = [0.4794,0.7794]$.

\item [b)] If $\alpha = 0.5$ then
\begin{equation*}
K_{0.5}(R)=Op(K_{0.5}([0.6, 0.9]),K_{0.5}([0.8, 0.8])) = 0.75 \cdot 0.8 = 0.6,
\end{equation*}
 and
\begin{eqnarray*}
w(R)= \min\left\{\max\left(\frac{0.3}{0.5},\frac{0}{0.4}\right) \cdot 0.8, 0.3\right\} = \min\{0.48,0.3\} = 0.3.
\end{eqnarray*}

Thus, $IOpw_{\max}^{0.5}([0.6,0.9],[0.8,0.8]) = [0.6 - 0.5 \cdot 0.3, 0.6 + 0.5 \cdot 0.3] = [0.45,0.75]$.

\item [c)] If $\alpha = 0.99$ then
\begin{equation*}
K_{0.99}(R)=Op(K_{0.99}([0.6, 0.9]),K_{0.99}([0.8, 0.8])) = 0.897 \cdot 0.8 = 0.7176,
\end{equation*}
 and
\begin{eqnarray*}
w(R)= \min\left\{\max\left(\frac{0.3}{0.906},\frac{0}{0.808}\right) \cdot 0.7248, 0.3\right\} = \min\{0.24,0.3\} = 0.24.
\end{eqnarray*}

Therefore, $IOpw_{\max}^{0.99}([0.6,0.9],[0.8,0.8]) = [0.7176 - 0.99 \cdot 0.24, 0.7176 + 0.01 \cdot 0.24] = [0.48,0.72]$.
\end{description}

\item [3)]  Finally, take $X = [0.6, 0.9]$, and $Y = [0.8,0.8]$, but consider $B = \min$. Then, we have that

\begin{equation*}
K_{\alpha}(R)=Op(K_{\alpha}([0.6, 0.9]),K_{\alpha}([0.8, 0.8]))
\end{equation*}
and, by \textbf{(ii)},
\begin{eqnarray*}
&&w(R) = \min\{\min(\lambda_{\alpha}([0.6, 0.9]),\lambda_{\alpha}([0.8, 0.8])) \cdot d_{\alpha}(Op(K_{\alpha}([0.6, 0.9]),K_{\alpha}([0.8, 0.8]))), 0\} = 0.
\end{eqnarray*}

So, let us see the different values of $IOpw_{\min}^{\alpha}([0.6, 0.9],[0.8, 0.8])$ in this case by varying the value of $\alpha$:

    \begin{description}
\item [a)] If $\alpha = 0.01$ then $IOpw_{\min}^{0.01}([0.6,0.9],[0.8,0.8]) = [0.4824, 0.4824]$;

\item [b)] If $\alpha = 0.5$ then $IOpw_{\min}^{0.5}([0.6,0.9],[0.8,0.8]) = [0.56, 0.56]$;

\item [c)] If $\alpha = 0.99$ then $IOpw_{\min}^{0.99}([0.6,0.9],[0.8,0.8]) = [0.7176, 0.7176]$.
\end{description}

\end{description}
\end{example}

\begin{remark}
The reason why $\alpha \in (0,1)$ is to assure that the construction method results in an w-iv-overlap function, so that conditions \textbf{(IOw2)} and \textbf{(IOw3)} are respected. Moreover, one may observe that the construction method presented in Theorem \ref{theo-conzdenko}, for a given overlap $O$, may not produce intervals contained in the best interval representation of $O$. However, it generates an interval-valued function which is $\leq_{\alpha,\beta}$-increasing and the chosen width-limiting aggregation function $B$ does not need to be conjunctive.  In the case when $B$ is conjunctive, as Remark \ref{rem-def} states, when either $X$ or $Y$ is degenerate, then $IOw_{B}^{\alpha}(X,Y)$ is also degenerate.
\end{remark}

Table \ref{table-compare} shows a comparison between the three construction methods for w-iv-overlap functions presented in Theorems \ref{teo-width} (Construction 1), \ref{theo-benja-geral} (Construction 2) and \ref{theo-conzdenko} (Construction 3), regarding some desirable properties (marked with $\checkmark$) and some possible drawbacks (marked with \xmark).

On Table \ref{table-ex}, we review the results obtained from Examples \ref{ex-con-w} and \ref{ex-con-zdenko}, to further compare the constructions methods presented on Theorems \ref{teo-width} (Construction 1) and \ref{theo-conzdenko} (Construction 3), all based on the product overlap $O_P$, but with different choices of the width-limiting function $B$ and different values of $\alpha$. As the construction method provided by Theorem \ref{theo-conzdenko} (Construction 3) does not allow for $\alpha = 1$, we present the values obtained by this method for $\alpha = 0.99$, instead. We omitted the results from Example \ref{ex-con-benja} on Table \ref{table-ex}, as the construction method based on Theorem \ref{theo-benja-geral} (Construction 2) presented itself as the most restrictive one, by a simple analysis of Table \ref{table-compare}.

\begin{table}[t]
	\caption{Comparison between construction methods of a w-iv-overlap $IOw_B^{\alpha}$, based on an overlap function $O$ and a width-limiting function $B$.}		\label{table-compare}
	\centering
	\begin{tabular}{lccc}
		\hline
         & Construction 1& Construction 2 & Construction 3 \\
         \hline 
         Advantages & &  &\\
		\hline 
		& & &\\[-0.2cm]
        $IOw_B^{\alpha}$ is ($\leq_{Pr},\leq_{\alpha,\beta}$)-increasing & $\checkmark$ & $\checkmark$ & $\checkmark$\\[0.2cm]
         $IOw_B^{\alpha}$ is $\leq_{\alpha,\beta}$-increasing & & $\checkmark$ & $\checkmark$\\[0.2cm]
          For all $X,Y \in L([0,1])$: $IOw_B^{\alpha}(X,Y) \subseteq \widehat{O}(X,Y)$  & $\checkmark$ &  & \\[0.2cm]
          \hline 
		& & &\\[-0.4cm]
Drawbacks & & & \\
		\hline \\[-0.2cm] 
$\alpha$ must be different than $1$ &  & \xmark & \xmark\\[0.1cm]
$\beta < \alpha$ must hold & \xmark &  & \\[0.1cm]
          $B$ needs to be conjunctive &  & \xmark & \\[0.2cm]
          For all $B$: &  & \xmark & \\[0.1cm]
          ($w(X)=0$ or $w(Y)=0$) $\Rightarrow IOw_B^{\alpha}(X,Y) = 0$ &  & & \\[0.1cm]
\hline 
		
	\end{tabular}
\end{table}

\begin{table}[t]
	\caption{Comparison of the results obtained in Examples \ref{ex-con-w} and \ref{ex-con-zdenko}}		\label{table-ex}
	\centering
	\begin{footnotesize}\begin{tabular}{lccc}
		\hline
         & Construction 1 & Construction 3 & Best Interval Representation\\
		\hline 
		& & &\\[-0.2cm]
        $X=[0.2,0.8]$&  &   & \\
        $Y=[0.5,1]$ & $IOwp_{\max}^{0.01}=[0.107,0.707]$ &  $IOwp_{\max}^{0.01}=[0.098,0.698]$ & $\widehat{O_P}(X,Y) = [0.1,0.8]$\\
        $A=\max$& & &  \\
        $\alpha = 0.01$ & &  & \\
         \hline 
		& &  &\\[-0.2cm]
           $X=[0.2,0.8]$&  &  & \\
        $Y=[0.5,1]$ & $IOwp_{\max}^{0.5}=[0.15,0.75]$ &  $IOwp_{\max}^{0.5}=[0.075,0.675]$ & $\widehat{O_P}(X,Y) = [0.1,0.8]$\\
        $A=\max$& & & \\
        $\alpha = 0.5$ & & & \\
        \hline 
		& &  &\\[-0.2cm]
           $X=[0.2,0.8]$&  &  & \\
        $Y=[0.5,1]$ & $IOwp_{\max}^{1}=[0.2,0.8]$ &  $IOwp_{\max}^{0.99}=[0.196,0.796]$ & $\widehat{O_P}(X,Y) = [0.1,0.8]$\\
        $A=\max$& & &  \\
        $\alpha = 1$ & &  & \\
        \hline 
		& & &\\[-0.2cm]
        $X=[0.6,0.9]$&  &  & \\
        $Y=[0.8,0.8]$ & $IOwp_{\max}^{0.01}=[0.48,0.72]$ &  $IOwp_{\max}^{0.01}=[0.4794, 0.7794]$ & $\widehat{O_P}(X,Y) = [0.48,0.72]$\\
        $A=\max$& & & \\
        $\alpha = 0.01$ & & &  \\
         \hline 
		& &  &\\[-0.2cm]
           $X=[0.6,0.9]$ &  &   & \\
        $Y=[0.8,0.8]$ & $IOwp_{\max}^{0.5}=[0.48,0.72]$ &  $IOwp_{\max}^{0.5}=[0.45,0.75]$ & $\widehat{O_P}(X,Y) = [0.48,0.72]$\\
        $A=\max$& & & \\
        $\alpha = 0.5$ & &  & \\
        \hline 
		& & & \\[-0.2cm]
           $X=[0.6,0.9]$&  &  & \\
        $Y=[0.8,0.8]$ & $IOwp_{\max}^{1}=[0.48,0.72]$ &  $IOwp_{\max}^{0.99}=[0.48,0.72]$ & $\widehat{O_P}(X,Y) = [0.48,0.72]$\\
        $A=\max$& & &  \\
        $\alpha = 1$ & & &  \\
       \hline 
		& & &\\[-0.2cm]
        $X=[0.6,0.9]$&  & &   \\
        $Y=[0.8,0.8]$ & $IOwp_{\min}^{0}=[0.4824,0.4824]$ & $IOwp_{\min}^{0.01}=[0.4824, 0.4824]$ & $\widehat{O_P}(X,Y) = [0.48,0.72]$\\
        $A=\min$& & &  \\
        $\alpha = 0.01$ & & & \\
         \hline 
		& & & \\[-0.2cm]
           $X=[0.6,0.9]$ &  &  & \\
        $Y=[0.8,0.8]$ & $IOwp_{\min}^{0.5}=[0.6,0.6]$ & $IOwp_{\min}^{0.5}=[0.56,0.56]$ & $\widehat{O_P}(X,Y) = [0.48,0.72]$\\
        $A=\min$& & &  \\
        $\alpha = 0.5$ & & & \\
        \hline 
		& & & \\[-0.2cm]
           $X=[0.6,0.9]$&  &   & \\
        $Y=[0.8,0.8]$ & $IOwp_{\min}^{1}=[0.72,0.72]$ & $IOwp_{\min}^{0.99}=[0.7176, 0.7176]$ & $\widehat{O_P}(X,Y) = [0.48,0.72]$\\
        $A=\min$& &  & \\
        $\alpha = 1$  & & & \\
\hline 
		
	\end{tabular}\end{footnotesize}
\end{table}

\begin{remark}\label{rem-aplicacao}
Concerning the application of the presented construction methods of width-limited iv-overlap functions in practical problems, a number of choices need to be made by the domain expert:
\begin{description}
  \item[1.] The choice of overlap function $O$: According to the considered application, some overlap functions produce better results than others. For example, in the literature, it is possible to verify that some overlap functions are more suitable to be applied in image processing \cite{Jurio201369} while others present good behaviour in classification problems \cite{elkano,edurne,ASMUS2020,ASMUS2020-TFS}.
  \item[2.] The choice of $\alpha$: It is completely determined by the admissible order $\leq_{\alpha,\beta}$ that is suitable for the application. The choice of the interval order depends on how the intervals are obtained or interpreted \cite{BUSTINCE201369,BENTKOWSKA2015792,BUSTINCE202023}.
  \item[3.] The choice of the width-limiting function $B$: Different applications may require that the aggregation process produces interval-valued outputs with more or less uncertainty tolerance, which will inform the definition of $B$. This will be determined by the relation between information accuracy and information quality required by the application. For example, when using $B = \max$ one has a better control of the accuracy of the result than when $B = \min$. However, the higher the accuracy, the lesser will be the information quality \cite{moore2009,ML03,DIM00}.
\end{description}
\end{remark}

\section{Conclusion} \label{sec-conclusion}

We introduced and developed the concepts of width-limited interval-valued functions and their respective width-limiting functions, as a way to analyze the effect of the width of the input intervals on the width of the output interval, accordingly to the interval-valued function at hand. Furthermore, it was shown a way to obtain the least width-limiting function for a given interval-valued function, which informs how much width-propagation one can expect for such interval-valued operation. A relaxation of the concept of ultramodularity was presented, in the form of $(a,b)$-ultramodular functions, allowing us to analyze the width-limiting functions of the best interval representation of some aggregation functions. Also, we introduced the notion of an interval-valued function that is increasing with respect to a pair of partial orders, a more flexible approach for increasingness of interval-valued functions.

These new developed concepts could aid the definition of different interval-valued functions with controlled width propagation. As our primary interest was to apply such notions on interval-valued overlap operations, width-limited interval-valued overlap functions were defined and studied. Following that, three construction methods for w-iv-overlap functions were presented, analyzed and compared. As these construction methods are all based on choices of overlap functions, width-limiting functions and admissible orders, it was made clear the adaptability of the developed concepts, as one can obtain an interval-valued overlap operations that best satisfy the restrictions of the context regarding the acceptable amount of width propagation and/or the ordering of intervals to be applied.

Thus, the contributions of this work aimed to address the gap in the literature regarding the analysis of the width of interval-valued functions, especially interval-valued overlap functions, while providing the initial theoretical tools to allow the application of similarly defined width-limited interval-valued functions in practical problems, where the increasing uncertainty associated with the widths of the operated intervals may be an obstacle to overcome, in order to maintain the information quality. On the near future, we intend to generalize adequately the presented theoretical approach to allow for applications in the context of interval-valued fuzzy rule-based classification systems.

\section*{Acknowledgment}
Supported by CNPq (307781/2016-0,	301618/2019-4), FAPERGS (19/ 2551-0001660) and the Spanish Ministry of Science and Technology (TIN2016-77356-P, PID2019-108392GB I00 (AEI/10.13039/501100011033)). The fifth author was supported by the Grant APVV-0052-18.

\appendix

\section{Proof of Theorem \ref{teo-width}} \label{ap-proofs}
\begin{proof}
Consider a commutative and increasing function $B: [0,1]^2 \rightarrow [0,1]$, a strict overlap function $O: [0,1]^2 \rightarrow [0,1]$ and take $\alpha \in (0,1]$, $\beta \in [0,\alpha)$. Observe that, for all $X,Y \in L([0,1])$:
\begin{description}
  \item[(i)] $K_{\alpha}(IOw_{B}^{\alpha}(X,Y)) = K_{\alpha}(\widehat{O}(X,Y))$; 
  \item[(ii)] $w(IOw_{B}^{\alpha}(X,Y)) = m_{\widehat{O},B}(X,Y) = \min\{w(\widehat{O}(X,Y)),B(w(X),w(Y))\}$.
\end{description}
So, it is immediate that $IOw_{B}^{\alpha}$ is well defined. Now, let us verify if $IOw_{B}^{\alpha}$ respects conditions \textbf{(IOw1)-(IOw5)} of Definition \ref{def-w-iv-ov}.

\vspace{0.5cm}

\textbf{(IOw1)} Immediate, as $O$ and $B$ are both commutative;

\vspace{0.5cm}

\textbf{(IOw2)} $(\Rightarrow)$ Suppose that there are $X,Y \in L([0,1])$ such that $IOw_{B}^{\alpha}(X,Y) = [0,0]$. Then, we have the following cases:

\begin{description}
  \item[1)] $m_{\widehat{O},B}(X,Y) = w(\widehat{O}(X,Y))$

 From Equations (\ref{eq-kalpha}) and (\ref{eq-iow}), it follows that:
  \begin{eqnarray*}
  \lefteqn{[K_{\alpha}(\widehat{O}(X,Y)) - \alpha \cdot w(\widehat{O}(X,Y)),  K_{\alpha}(\widehat{O}(X,Y)) + (1 - \alpha)\cdot w(\widehat{O}(X,Y))] = [0,0]}&&\\
 &\Rightarrow& [O(\underline{X},\underline{Y}) + \alpha \cdot w(\widehat{O}(X,Y)) - \alpha \cdot w(\widehat{O}(X,Y)),\\
  &&O(\underline{X},\underline{Y}) + \alpha \cdot w(\widehat{O}(X,Y)) + w(\widehat{O}(X,Y)) - \alpha \cdot w(\widehat{O}(X,Y))] = [0,0]\\
 &\Rightarrow& [O(\underline{X},\underline{Y}), O(\underline{X},\underline{Y}) +  w(\widehat{O}(X,Y))] = [0,0] \Rightarrow [O(\underline{X},\underline{Y}), O(\overline{X},\overline{Y})] = [0,0]\\
  &\Rightarrow& \widehat{O}(X,Y) = [0,0] \Leftrightarrow X \cdot Y = [0,0].
  \end{eqnarray*}

  \item[2)] $m_{\widehat{O},B}(X,Y) = B(w(X),w(Y))$

  From Equations (\ref{eq-kalpha}) and (\ref{eq-iow}), it holds that:
  \begin{eqnarray*}
 \lefteqn{[K_{\alpha}(\widehat{O}(X,Y)) - \alpha \cdot B(w(X),w(Y)),  K_{\alpha}(\widehat{O}(X,Y)) + (1 - \alpha)\cdot B(w(X),w(Y))] = [0,0]}&&\\
&\Rightarrow&  - \alpha \cdot B(w(X),w(Y)) = (1 - \alpha)\cdot B(w(X),w(Y)) \Rightarrow B(w(X),w(Y)) = 0\\
 &\Rightarrow& [K_{\alpha}(\widehat{O}(X,Y)),K_{\alpha}(\widehat{O}(X,Y))] = [0,0] \Rightarrow K_{\alpha}(\widehat{O}(X,Y)) = 0\\ 
 &\Rightarrow& \widehat{O}(X,Y) = [0,0] \Leftrightarrow X \cdot Y = [0,0].
  \end{eqnarray*}
\end{description}

$(\Leftarrow)$ Consider $X,Y \in L([0,1])$ such that $X \cdot Y = [0,0]$. Then, it is immediate that $\widehat{O}(X,Y) = [0,0]$ and $m_{\widehat{O},B}(X,Y) = 0$. Furthermore, from Equation (\ref{eq-iow}):

\begin{equation*}
IOw_{B}^{\alpha}(X,Y)= [K_{\alpha}([0,0]) - \alpha \cdot 0,  K_{\alpha}([0,0]) + (1 - \alpha)\cdot 0] = [0,0].
\end{equation*}

\vspace{0.5cm}

\textbf{(IOw3)} $(\Rightarrow)$ Consider $X,Y \in L([0,1])$ such that $IOw_{B}^{\alpha}(X,Y) = [1,1]$. Then, we have the following cases:

\begin{description}
  \item[1)] $m_{\widehat{O},B}(X,Y) = w(\widehat{O}(X,Y))$

 From Equations (\ref{eq-kalpha}) and (\ref{eq-iow}), it follows that:
  \begin{eqnarray*}
  \lefteqn{[K_{\alpha}(\widehat{O}(X,Y)) - \alpha \cdot w(\widehat{O}(X,Y)),  K_{\alpha}(\widehat{O}(X,Y)) + (1 - \alpha)\cdot w(\widehat{O}(X,Y))] = [1,1]}&&\\
 &\Rightarrow& [O(\underline{X},\underline{Y}) + \alpha \cdot w(\widehat{O}(X,Y)) - \alpha \cdot w(\widehat{O}(X,Y)),\\
  &&O(\underline{X},\underline{Y}) + \alpha \cdot w(\widehat{O}(X,Y)) + w(\widehat{O}(X,Y)) - \alpha \cdot w(\widehat{O}(X,Y))] = [1,1]\\
 &\Rightarrow& [O(\underline{X},\underline{Y}), O(\underline{X},\underline{Y}) +  w(\widehat{O}(X,Y))] = [1,1] \Rightarrow [O(\underline{X},\underline{Y}), O(\overline{X},\overline{Y})] = [1,1]\\
  &\Rightarrow& \widehat{O}(X,Y) = [1,1] \Leftrightarrow X \cdot Y = [1,1].
  \end{eqnarray*}

  \item[2)] $m_{\widehat{O},B}(X,Y) = B(w(X),w(Y))$

  From Equations (\ref{eq-kalpha}) and (\ref{eq-iow}), it holds that:
  \begin{eqnarray*}
  \lefteqn{[K_{\alpha}(\widehat{O}(X,Y)) - \alpha \cdot B(w(X),w(Y)),  K_{\alpha}(\widehat{O}(X,Y)) + (1 - \alpha)\cdot B(w(X),w(Y))] = [1,1]}&&\\
 &\Rightarrow&  - \alpha \cdot B(w(X),w(Y)) = (1 - \alpha)\cdot B(w(X),w(Y)) \Rightarrow B(w(X),w(Y)) = 0\\
 &\Rightarrow& [K_{\alpha}(\widehat{O}(X,Y)),K_{\alpha}(\widehat{O}(X,Y))] = [1,1] \Rightarrow K_{\alpha}(\widehat{O}(X,Y)) = 1\\ 
 &\Rightarrow& \widehat{O}(X,Y) = [1,1] \Leftrightarrow X \cdot Y = [1,1].
  \end{eqnarray*}
\end{description}

$(\Leftarrow)$ Consider $X,Y \in L([0,1])$ such that $X \cdot Y = [1,1]$. Then, it is immediate that $\widehat{O}(X,Y) = [1,1]$ and $m_{\widehat{O},B}(X,Y) = 0$. Furthermore, from Equation (\ref{eq-iow}):

\begin{equation*}
IOw_{B}^{\alpha}(X,Y)= [K_{\alpha}([1,1]) - \alpha \cdot 0,  K_{\alpha}([1,1]) + (1 - \alpha)\cdot 0] = [1,1].
\end{equation*}

\vspace{0.5cm}

\textbf{(IOw4)} Consider $X,Y,Z \in L([0,1])$ such that $X \leq_{Pr} Y$. Then:
\begin{equation}\label{iowax}
IOw_{B}^{\alpha}(X,Z)= [K_{\alpha}(\widehat{O}(X,Z)) - \alpha \cdot m_{\widehat{O},B}(X,Z),  K_{\alpha}(\widehat{O}(X,Z)) + (1 - \alpha)\cdot m_{\widehat{O},B}(X,Z)],
\end{equation}
and
\begin{equation}\label{ioway}
IOw_{B}^{\alpha}(Y,Z)= [K_{\alpha}(\widehat{O}(Y,Z)) - \alpha \cdot m_{\widehat{O},B}(Y,Z),  K_{\alpha}(\widehat{O}(Y,Z)) + (1 - \alpha)\cdot m_{\widehat{O},B}(Y,Z)].
\end{equation}

Observe that $IOw_{B}^{\alpha}(X,Z)$ is obtained by constructing an interval around the value of $K_{\alpha}(\widehat{O}(X,Z))$, and that $\widehat{O}(X,Z)$ is an $o$-representable iv-overlap function with $O$ as both its representatives. Then, from Equations (\ref{eq-kalpha}) and (\ref{iowax}), it follows that:
\begin{eqnarray}
&&K_{\alpha}(IOw_{B}^{\alpha}(X,Z)) = K_{\alpha}(\widehat{O}(X,Z)),\label{kalfax}\\
\nonumber &&K_{\beta}(IOw_{B}^{\alpha}(X,Z)) = K_{\alpha}(\widehat{O}(X,Z)) - \alpha \cdot m_{\widehat{O},B}(X,Z) + \beta \cdot m_{\widehat{O},B}(X,Z).
\end{eqnarray}
As $\beta<\alpha$, by Lemma \ref{lemma-novo}, one can consider $\beta = 0$. Thus, we have that:
\begin{equation}\label{kbetax}
K_{\beta}(IOw_{B}^{\alpha}(X,Z)) = K_{\alpha}(\widehat{O}(X,Z)) - \alpha \cdot m_{\widehat{O},B}(X,Z).
\end{equation}

Analogously, from Equations (\ref{eq-kalpha}) and (\ref{ioway}), it follows that:
\begin{eqnarray}
&&K_{\alpha}(IOw_{B}^{\alpha}(Y,Z)) = K_{\alpha}(\widehat{O}(Y,Z)),\label{kalfay}\\
&&K_{\beta}(IOw_{B}^{\alpha}(Y,Z)) = K_{\alpha}(\widehat{O}(Y,Z)) - \alpha \cdot m_{\widehat{O},B}(Y,Z)\label{kbetay}.
\end{eqnarray}

Now, we have the following possibilities regarding $m_{\widehat{O},B}(X,Z)$ and $m_{\widehat{O},B}(Y,Z)$ that affects the values of $IOw_{B}^{\alpha}(X,Z)$ and $IOw_{B}^{\alpha}(Y,Z)$, respectively:

\begin{description}

\item [1)] $m_{\widehat{O},B}(X,Z)=w(\widehat{O}(X,Z))$ and $m_{\widehat{O},B}(Y,Z)=w(\widehat{O}(Y,Z))$

In this case, we have
\begin{equation*}
IOw_{B}^{\alpha}(X,Z) = \widehat{O}(X,Z) \leq_{Pr} \widehat{O}(Y,Z) = IOw_{B}^{\alpha}(Y,Z),
\end{equation*}
meaning that $IOw_{B}^{\alpha}(X,Z) \leq_{\alpha,\beta} IOw_{B}^{\alpha}(Y,Z)$.

\item[2)] $m_{\widehat{O},B}(X,Z)=B(w(X),w(Z))$ and $m_{\widehat{O},B}(Y,Z)=B(w(Y),w(Z))$

It follows that
\begin{equation*}
IOw_{B}^{\alpha}(X,Z)= [K_{\alpha}(\widehat{O}(X,Z)) - \alpha \cdot B(w(X),w(Z)),  K_{\alpha}(\widehat{O}(X,Z)) + (1 - \alpha)\cdot B(w(X),w(Z))],
\end{equation*}
and
\begin{equation*}
IOw_{B}^{\alpha}(Y,Z)= [K_{\alpha}(\widehat{O}(Y,Z)) - \alpha \cdot B(w(Y),w(Z)),  K_{\alpha}(\widehat{O}(Y,Z)) + (1 - \alpha)\cdot B(w(Y),w(Z))].
\end{equation*}

Now, let us verify all the cases in which $X \leq_{Pr} Y$ holds:

\begin{description}

\item [a)] $\underline{X} = \underline{Y}$ and  $\overline{X} = \overline{Y}$:

We have that $X = Y$, meaning that \[IOw_{B}^{\alpha}(X,Z) = IOw_{B}^{\alpha}(Y,Z) \Rightarrow IOw_{B}^{\alpha}(X,Z) \leq_{\alpha,\beta} IOw_{B}^{\alpha}(Y,Z).\]

\item[b)] $\underline{X} = \underline{Y}$ and  $\overline{X} < \overline{Y}$:

When $\underline{Z} \neq 0$, from Lemma \ref{lema-prkalfa}, it holds that $K_{\alpha}(\widehat{O}(X,Z)) < K_{\alpha}(\widehat{O}(Y,Z))$, since $O$ is a strict overlap function and $\alpha \in (0,1]$. As $K_{\alpha}(IOw_{B}^{\alpha}(X,Z)) = K_{\alpha}(\widehat{O}(X,Z))$ and $K_{\alpha}(IOw_{B}^{\alpha}(Y,Z)) = K_{\alpha}(\widehat{O}(Y,Z))$, we have that \[K_{\alpha}(IOw_{B}^{\alpha}(X,Z)) < K_{\alpha}(IOw_{B}^{\alpha}(Y,Z)) \Rightarrow IOw_{B}^{\alpha}(X,Z) \leq_{\alpha,\beta} IOw_{B}^{\alpha}(Y,Z).\] 

If $\underline{Z} = 0$ and $\overline{Z} \neq 0$, by \textbf{(O2)}, one has that
\begin{equation*}
\widehat{O}(X,Z)= [0, O(\overline{X},\overline{Z})],
\end{equation*}
and
\begin{equation*}
\widehat{O}(Y,Z)= [0, O(\overline{Y},\overline{Z})].
\end{equation*}
Since $\overline{X} < \overline{Y}$ and $O$ is strict, then
\begin{eqnarray*}
\lefteqn{K_{\alpha}(IOw_{B}^{\alpha}(X,Z)) = K_{\alpha}(\widehat{O}(X,Z)) < K_{\alpha}(\widehat{O}(Y,Z)) = K_{\alpha}(IOw_{B}^{\alpha}(Y,Z))}&&\\
&\Rightarrow &IOw_{B}^{\alpha}(X,Z) \leq_{\alpha,\beta} IOw_{B}^{\alpha}(Y,Z).
\end{eqnarray*}

If $\underline{Z} = 0$ and $\overline{Z} = 0$, then
\begin{equation*}
\widehat{O}(X,Z)=IOw_{B}^{\alpha}(X,Z)=[0,0]=IOw_{B}^{\alpha}(Y,Z)=\widehat{O}(X,Z).
\end{equation*}

So, we have that $IOw_{B}^{\alpha}(X,Z) \leq_{\alpha,\beta} IOw_{B}^{\alpha}(Y,Z)$, for all $X,Y,Z \in L([0,1])$, such that $\underline{X} = \underline{Y}$ and  $\overline{X} < \overline{Y}$.

\item[c)] $\underline{X} < \underline{Y}$ and  $\overline{X} = \overline{Y}$:

When $\underline{Z} \neq 0$ and $\alpha \neq 1$, from Lemma \ref{lema-prkalfa}, we have that $K_{\alpha}(\widehat{O}(X,Z)) < K_{\alpha}(\widehat{O}(Y,Z))$. So, it holds that
\begin{equation*}
K_{\alpha}(IOw_{B}^{\alpha}(X,Z)) < K_{\alpha}(IOw_{B}^{\alpha}(Y,Z)) \Rightarrow IOw_{B}^{\alpha}(X,Z) \leq_{\alpha,\beta} IOw_{B}^{\alpha}(Y,Z).
\end{equation*}

When taking $\underline{Z} \neq 0$ and $\alpha = 1$, we have that $K_{\alpha}(IOw_{B}^{\alpha}(X,Z)) = K_{\alpha}(IOw_{B}^{\alpha}(Y,Z))$. Moreover, from  Equations (\ref{kbetax}) and (\ref{kbetay}):
\begin{equation*}
K_{\beta}(IOw_{B}^{1}(X,Z)) = O(\overline{X},\overline{Z}) - B(w(X),w(Z))
\end{equation*}

and
\begin{equation*}
K_{\beta}(IOw_{B}^{1}(Y,Z)) = O(\overline{Y},\overline{Z}) - B(w(Y),w(Z)).
\end{equation*}

As $\underline{X} < \underline{Y}$ and $\overline{X} = \overline{Y}$, we have that $w(Y) < w(X)$, and thus, $B(w(Y),w(Z)) \leq B(w(X),w(Z))$, as $B$ is increasing. So,
\begin{equation*}
K_{\beta}(IOw_{B}^{1}(X,Z)) = O(\overline{X},\overline{Z}) - B(w(X),w(Z)) \leq O(\overline{Y},\overline{Z}) - B(w(Y),w(Z)) = K_{\beta}(IOw_{B}^{1}(Y,Z)).
\end{equation*}
Then,
\begin{eqnarray*}
\lefteqn{K_{\alpha}(IOw_{B}^{\alpha}(X,Z)) = K_{\alpha}(IOw_{B}^{\alpha}(Y,Z)) \, \, \mbox{and} \, \, K_\beta(IOw_{B}^{\alpha}(X,Z)) \leq K_\beta(IOw_{B}^{\alpha}(Y,Z))}&&\\
&\Rightarrow& IOw_{B}^{\alpha}(X,Z) \leq_{\alpha,\beta} IOw_{B}^{\alpha}(Y,Z).
\end{eqnarray*}

If $\underline{Z} = 0$, by \textbf{(O2)}, one has that
\begin{equation*}
\widehat{O}(X,Z)= [0, O(\overline{X},\overline{Z})],
\end{equation*}
and
\begin{equation*}
\widehat{O}(Y,Z)= [0, O(\overline{Y},\overline{Z})].
\end{equation*}
Since $\overline{X} = \overline{Y}$, then $K_{\alpha}(IOw_{B}^{\alpha}(X,Z)) = K_{\alpha}(IOw_{B}^{\alpha}(Y,Z))$ and, analogous to the previous case when $\underline{Z} \neq 0$ and $\alpha = 1$, we have that
\begin{eqnarray*}
\lefteqn{K_{\alpha}(IOw_{B}^{\alpha}(X,Z)) = K_{\alpha}(IOw_{B}^{\alpha}(Y,Z)) \, \, \mbox{and} \, \, K_\beta(IOw_{B}^{\alpha}(X,Z)) \leq K_\beta(IOw_{B}^{\alpha}(Y,Z))}&&\\
&\Rightarrow& IOw_{B}^{\alpha}(X,Z) \leq_{\alpha,\beta} IOw_{B}^{\alpha}(Y,Z).
\end{eqnarray*}

So, we have that $IOw_{B}^{\alpha}(X,Z) \leq_{\alpha,\beta} IOw_{B}^{\alpha}(Y,Z)$, for all $X,Y,Z \in L([0,1])$, such that $\underline{X} < \underline{Y}$ and  $\overline{X} = \overline{Y}$.

\item[d)] $\underline{X} < \underline{Y}$ and  $\overline{X} < \overline{Y}$:

When $\underline{Z} \neq 0$, from Lemma \ref{lema-prkalfa}, it holds that $K_{\alpha}(\widehat{O}(X,Z)) < K_{\alpha}(\widehat{O}(Y,Z))$. So, we have that
\begin{equation*}
K_{\alpha}(IOw_{B}^{\alpha}(X,Z)) < K_{\alpha}(IOw_{B}^{\alpha}(Y,Z)) \Rightarrow IOw_{B}^{\alpha}(X,Z) \leq_{\alpha,\beta} IOw_{B}^{\alpha}(Y,Z).
\end{equation*}

If $\underline{Z} = 0$ and $\overline{Z} \neq 0$, by \textbf{(O2)}, one has that
\begin{equation*}
\widehat{O}(X,Z)= [0, O(\overline{X},\overline{Z})],
\end{equation*}
and
\begin{equation*}
\widehat{O}(Y,Z)= [0, O(\overline{Y},\overline{Z})].
\end{equation*}
Since $\overline{X} < \overline{Y}$ and $O$ is strict, then
\begin{eqnarray*}
\lefteqn{K_{\alpha}(IOw_{B}^{\alpha}(X,Z)) = K_{\alpha}(\widehat{O}(X,Z)) < K_{\alpha}(\widehat{O}(Y,Z)) = K_{\alpha}(IOw_{B}^{\alpha}(Y,Z))}&&\\
&\Rightarrow &IOw_{B}^{\alpha}(X,Z) \leq_{\alpha,\beta} IOw_{B}^{\alpha}(Y,Z).
\end{eqnarray*}

If $\underline{Z} = 0$ and $\overline{Z} = 0$, then
\begin{equation*}
\widehat{O}(X,Z)=IOw_{B}^{\alpha}(X,Z)=[0,0]=IOw_{B}^{\alpha}(Y,Z)=\widehat{O}(X,Z).
\end{equation*}
So, we have that $IOw_{B}^{\alpha}(X,Z) \leq_{\alpha,\beta} IOw_{B}^{\alpha}(Y,Z)$, for all $X,Y,Z \in L([0,1])$, such that $\underline{X} < \underline{Y}$ and  $\overline{X} < \overline{Y}$.
\end{description}

Thus, one can conclude that, for all $X,Y,Z \in L([0,1])$, when $m_{\widehat{O},B}(X,Z)=B(w(X),w(Z))$ and $m_{\widehat{O},B}(Y,Z)=B(w(Y),w(Z))$, then
\begin{equation*}
X\leq_{Pr}Y \Rightarrow IOw_{B}^{\alpha}(X,Z) \leq_{\alpha,\beta} IOw_{B}^{\alpha}(Y,Z).
\end{equation*}

\item[3)] $m_{\widehat{O},B}(X,Z)=w(\widehat{O}(X,Z))$ and $m_{\widehat{O},B}(Y,Z)=B(w(Y),w(Z))$

It follows that
\begin{equation*}
IOw_{B}^{\alpha}(X,Z)= \widehat{O}(X,Z),
\end{equation*}
and
\begin{equation*}
IOw_{B}^{\alpha}(Y,Z)= [K_{\alpha}(\widehat{O}(Y,Z)) - \alpha \cdot B(w(Y),w(Z)),  K_{\alpha}(\widehat{O}(Y,Z)) + (1 - \alpha)\cdot B(w(Y),w(Z))].
\end{equation*}

Now, let us verify all the cases in which $X \leq_{Pr} Y$ holds:

\begin{description}

\item [a)] $\underline{X} = \underline{Y}$ and  $\overline{X} = \overline{Y}$:

We have that $X = Y$ and
\begin{equation*}
IOw_{B}^{\alpha}(X,Z) = IOw_{B}^{\alpha}(Y,Z) \Rightarrow IOw_{B}^{\alpha}(X,Z) \leq_{\alpha,\beta} IOw_{B}^{\alpha}(Y,Z).
\end{equation*}

\item[b)] $\underline{X} = \underline{Y}$ and  $\overline{X} < \overline{Y}$:

When $\underline{Z}\neq 0$, from Lemma \ref{lema-prkalfa}, it holds that $K_{\alpha}(\widehat{O}(X,Z)) < K_{\alpha}(\widehat{O}(Y,Z))$, since $O$ is a strict overlap function and $\alpha \in (0,1]$. So, as $K_{\alpha}(IOw_{B}^{\alpha}(X,Z)) = K_{\alpha}(\widehat{O}(X,Z))$ and $K_{\alpha}(IOw_{B}^{\alpha}(Y,Z)) = K_{\alpha}(\widehat{O}(Y,Z))$, we have that
\begin{equation*}
K_{\alpha}(IOw_{B}^{\alpha}(X,Z)) < K_{\alpha}(IOw_{B}^{\alpha}(Y,Z)) \Rightarrow IOw_{B}^{\alpha}(X,Z) \leq_{\alpha,\beta} IOw_{B}^{\alpha}(Y,Z).
\end{equation*}

If $\underline{Z} = 0$ and $\overline{Z} \neq 0$, by \textbf{(O2)}, one has that
\begin{equation*}
\widehat{O}(X,Z)= [0, O(\overline{X},\overline{Z})],
\end{equation*}
and
\begin{equation*}
\widehat{O}(Y,Z)= [0, O(\overline{Y},\overline{Z})].
\end{equation*}
Since $\overline{X} < \overline{Y}$ and $O$ is strict, then
\begin{eqnarray*}
\lefteqn{K_{\alpha}(IOw_{B}^{\alpha}(X,Z)) = K_{\alpha}(\widehat{O}(X,Z)) < K_{\alpha}(\widehat{O}(Y,Z)) = K_{\alpha}(IOw_{B}^{\alpha}(Y,Z))}&&\\
&\Rightarrow &IOw_{B}^{\alpha}(X,Z) \leq_{\alpha,\beta} IOw_{B}^{\alpha}(Y,Z).
\end{eqnarray*}

If $\underline{Z} = 0$ and $\overline{Z} = 0$, then
\begin{equation*}
\widehat{O}(X,Z)=IOw_{B}^{\alpha}(X,Z)=[0,0]=IOw_{B}^{\alpha}(Y,Z)=\widehat{O}(X,Z).
\end{equation*}

So, we have that $IOw_{B}^{\alpha}(X,Z) \leq_{\alpha,\beta} IOw_{B}^{\alpha}(Y,Z)$, for all $X,Y,Z \in L([0,1])$, such that $\underline{X} = \underline{Y}$ and  $\overline{X} < \overline{Y}$.

\item[c)] $\underline{X} < \underline{Y}$ and  $\overline{X} = \overline{Y}$:

When $\underline{Z} \neq 0$ and $\alpha \neq 1$, from Lemma \ref{lema-prkalfa}, we have that $K_{\alpha}(\widehat{O}(X,Z)) < K_{\alpha}(\widehat{O}(Y,Z))$. So, it holds that
\begin{equation*}
K_{\alpha}(IOw_{B}^{\alpha}(X,Z)) < K_{\alpha}(IOw_{B}^{\alpha}(Y,Z)) \Rightarrow IOw_{B}^{\alpha}(X,Z) \leq_{\alpha,\beta} IOw_{B}^{\alpha}(Y,Z).
\end{equation*}

If $\underline{Z} \neq 0$ and $\alpha = 1$, we have that $K_{\alpha}(IOw_{B}^{\alpha}(X,Z)) = K_{\alpha}(IOw_{B}^{\alpha}(Y,Z))$. Moreover, from  Equations (\ref{kbetax}) and (\ref{kbetay}):
\begin{equation*}
K_{\beta}(IOw_{B}^{1}(X,Z)) = O(\overline{X},\overline{Z}) - w(\widehat{O}(X,Z))
\end{equation*}

and
\begin{equation*}
K_{\beta}(IOw_{B}^{1}(Y,Z)) = O(\overline{Y},\overline{Z}) - B(w(Y),w(Z)).
\end{equation*}

As $\underline{X} < \underline{Y}$ and $\overline{X} = \overline{Y}$,  we have that
\begin{equation*}
B(w(Y),w(Z)) \leq  w(\widehat{O}(Y,Z)) =  O(\overline{Y},\overline{Z}) - O(\underline{Y},\underline{Z}) \leq O(\overline{X},\overline{Z}) - O(\underline{X},\underline{Z}) = w(\widehat{O}(X,Z)),
\end{equation*}
as $O$ is increasing. So,
\begin{equation*}
K_{\beta}(IOw_{B}^{1}(X,Z)) = O(\overline{X},\overline{Z}) - w(\widehat{O}(X,Z)) \leq O(\overline{Y},\overline{Z}) - B(w(Y),w(Z)) = K_{\beta}(IOw_{B}^{1}(Y,Z)).
\end{equation*}
Then,
\begin{eqnarray*}
\lefteqn{K_{\alpha}(IOw_{B}^{\alpha}(X,Z)) = K_{\alpha}(IOw_{B}^{\alpha}(Y,Z)) \, \, \mbox{and} \, \, K_\beta(IOw_{B}^{\alpha}(X,Z)) \leq K_\beta(IOw_{B}^{\alpha}(Y,Z))}&&\\
&\Rightarrow& IOw_{B}^{\alpha}(X,Z) \leq_{\alpha,\beta} IOw_{B}^{\alpha}(Y,Z).
\end{eqnarray*}

When $\underline{Z} = 0$, by \textbf{(O2)} we have that

\begin{equation*}
\widehat{O}(X,Z)= [0, O(\overline{X},\overline{Z})],
\end{equation*}
and
\begin{equation*}
\widehat{O}(Y,Z)= [0, O(\overline{Y},\overline{Z})].
\end{equation*}

Since $\overline{X} = \overline{Y}$, then $K_{\alpha}(IOw_{B}^{\alpha}(X,Z)) = K_{\alpha}(IOw_{B}^{\alpha}(Y,Z))$ and, analogous to the previous case when $\underline{Z} \neq 0$ and $\alpha = 1$, we have that
\begin{eqnarray*}
\lefteqn{K_{\alpha}(IOw_{B}^{\alpha}(X,Z)) = K_{\alpha}(IOw_{B}^{\alpha}(Y,Z)) \, \, \mbox{and} \, \, K_\beta(IOw_{B}^{\alpha}(X,Z)) \leq K_\beta(IOw_{B}^{\alpha}(Y,Z))}&&\\
&\Rightarrow& IOw_{B}^{\alpha}(X,Z) \leq_{\alpha,\beta} IOw_{B}^{\alpha}(Y,Z).
\end{eqnarray*}

So, we have that $IOw_{B}^{\alpha}(X,Z) \leq_{\alpha,\beta} IOw_{B}^{\alpha}(Y,Z)$, for all $X,Y,Z \in L([0,1])$, such that $\underline{X} < \underline{Y}$ and  $\overline{X} = \overline{Y}$

\item[d)] $\underline{X} < \underline{Y}$ and  $\overline{X} < \overline{Y}$:

If $\underline{Z} \neq 0$, from Lemma \ref{lema-prkalfa}, it holds that $K_{\alpha}(\widehat{O}(X,Z)) < K_{\alpha}(\widehat{O}(Y,Z))$. So, we have that
\begin{equation*}
K_{\alpha}(IOw_{B}^{\alpha}(X,Z)) < K_{\alpha}(IOw_{B}^{\alpha}(Y,Z)) \Rightarrow IOw_{B}^{\alpha}(X,Z) \leq_{\alpha,\beta} IOw_{B}^{\alpha}(Y,Z).
\end{equation*}

If $\underline{Z} = 0$ and $\overline{Z} \neq 0$, by \textbf{(O2)}, one has that
\begin{equation*}
\widehat{O}(X,Z)= [0, O(\overline{X},\overline{Z})],
\end{equation*}
and
\begin{equation*}
\widehat{O}(Y,Z)= [0, O(\overline{Y},\overline{Z})].
\end{equation*}
Since $\overline{X} < \overline{Y}$ and $O$ is strict, then
\begin{eqnarray*}
\lefteqn{K_{\alpha}(IOw_{B}^{\alpha}(X,Z)) = K_{\alpha}(\widehat{O}(X,Z)) < K_{\alpha}(\widehat{O}(Y,Z)) = K_{\alpha}(IOw_{B}^{\alpha}(Y,Z))}&&\\
&\Rightarrow &IOw_{B}^{\alpha}(X,Z) \leq_{\alpha,\beta} IOw_{B}^{\alpha}(Y,Z).
\end{eqnarray*}

If $\underline{Z} = 0$ and $\overline{Z} = 0$, then
\begin{equation*}
\widehat{O}(X,Z)=IOw_{B}^{\alpha}(X,Z)=[0,0]=IOw_{B}^{\alpha}(Y,Z)=\widehat{O}(X,Z).
\end{equation*}
So, we have that $IOw_{B}^{\alpha}(X,Z) \leq_{\alpha,\beta} IOw_{B}^{\alpha}(Y,Z)$, for all $X,Y,Z \in L([0,1])$, such that $\underline{X} < \underline{Y}$ and  $\overline{X} < \overline{Y}$.
\end{description}

Thus, one can conclude that, for all $X,Y,Z \in L([0,1])$, when $m_{\widehat{O},B}(X,Z)=w(\widehat{O}(X,Z))$ and $m_{\widehat{O},B}(Y,Z)=B(w(Y),w(Z))$, then
\begin{equation*}
X\leq_{Pr}Y \Rightarrow IOw_{B}^{\alpha}(X,Z) \leq_{\alpha,\beta} IOw_{B}^{\alpha}(Y,Z).
\end{equation*}

\item[4)] $m_{\widehat{O},B}(X,Z)=B(w(X),w(Z))$ and $m_{\widehat{O},B}(Y,Z)=w(\widehat{O}(Y,Z))$

It follows that
\begin{equation*}
IOw_{B}^{\alpha}(X,Z)= [K_{\alpha}(\widehat{O}(X,Z)) - \alpha \cdot B(w(X),w(Z)),  K_{\alpha}(\widehat{O}(X,Z)) + (1 - \alpha)\cdot B(w(X),w(Z))],
\end{equation*}
and
\begin{equation*}
IOw_{B}^{\alpha}(Y,Z)= \widehat{O}(Y,Z).
\end{equation*}

Now, let us verify all the cases in which $X \leq_{Pr} Y$ holds:

\begin{description}

\item [a)] $\underline{X} = \underline{Y}$ and  $\overline{X} = \overline{Y}$:

We have that $X = Y$ and $IOw_{B}^{\alpha}(X,Z) = IOw_{B}^{\alpha}(Y,Z) \Rightarrow IOw_{B}^{\alpha}(X,Z) \leq_{\alpha,\beta} IOw_{B}^{\alpha}(Y,Z)$.

\item[b)] $\underline{X} = \underline{Y}$ and  $\overline{X} < \overline{Y}$:

When $\underline{Z}\neq 0$, from Lemma \ref{lema-prkalfa}, it holds that $K_{\alpha}(\widehat{O}(X,Z)) < K_{\alpha}(\widehat{O}(Y,Z))$, since $O$ is a strict overlap function and $\alpha \in (0,1]$. So, as $K_{\alpha}(IOw_{B}^{\alpha}(X,Z)) = K_{\alpha}(\widehat{O}(X,Z))$ and $K_{\alpha}(IOw_{B}^{\alpha}(Y,Z)) = K_{\alpha}(\widehat{O}(Y,Z))$, we have that
\begin{equation*}
K_{\alpha}(IOw_{B}^{\alpha}(X,Z)) < K_{\alpha}(IOw_{B}^{\alpha}(Y,Z)) \Rightarrow IOw_{B}^{\alpha}(X,Z) \leq_{\alpha,\beta} IOw_{B}^{\alpha}(Y,Z).
\end{equation*}

If $\underline{Z} = 0$ and $\overline{Z} \neq 0$, by \textbf{(O2)}, one has that
\begin{equation*}
\widehat{O}(X,Z)= [0, O(\overline{X},\overline{Z})],
\end{equation*}
and
\begin{equation*}
\widehat{O}(Y,Z)= [0, O(\overline{Y},\overline{Z})].
\end{equation*}
Since $\overline{X} < \overline{Y}$ and $O$ is strict, then
\begin{eqnarray*}
\lefteqn{K_{\alpha}(IOw_{B}^{\alpha}(X,Z)) = K_{\alpha}(\widehat{O}(X,Z)) < K_{\alpha}(\widehat{O}(Y,Z)) = K_{\alpha}(IOw_{B}^{\alpha}(Y,Z))}&&\\
&\Rightarrow &IOw_{B}^{\alpha}(X,Z) \leq_{\alpha,\beta} IOw_{B}^{\alpha}(Y,Z).
\end{eqnarray*}

If $\underline{Z} = 0$ and $\overline{Z} = 0$, then
\begin{equation*}
\widehat{O}(X,Z)=IOw_{B}^{\alpha}(X,Z)=[0,0]=IOw_{B}^{\alpha}(Y,Z)=\widehat{O}(X,Z).
\end{equation*}

So, we have that $IOw_{B}^{\alpha}(X,Z) \leq_{\alpha,\beta} IOw_{B}^{\alpha}(Y,Z)$, for all $X,Y,Z \in L([0,1])$, such that $\underline{X} = \underline{Y}$ and  $\overline{X} < \overline{Y}$.

\item[c)] $\underline{X} < \underline{Y}$ and  $\overline{X} = \overline{Y}$:

When $\underline{Z} \neq 0$ and $\alpha \neq 1$, from Lemma \ref{lema-prkalfa}, we have that $K_{\alpha}(\widehat{O}(X,Z)) < K_{\alpha}(\widehat{O}(Y,Z))$. So, it holds that
\begin{equation*}
K_{\alpha}(IOw_{B}^{\alpha}(X,Z)) < K_{\alpha}(IOw_{B}^{\alpha}(Y,Z)) \Rightarrow IOw_{B}^{\alpha}(X,Z) \leq_{\alpha,\beta} IOw_{B}^{\alpha}(Y,Z).
\end{equation*}

If $\underline{Z} \neq 0$ and $\alpha = 1$, we have that $K_{\alpha}(IOw_{B}^{\alpha}(X,Z)) = K_{\alpha}(IOw_{B}^{\alpha}(Y,Z))$. Moreover, from  Equations (\ref{kbetax}) and (\ref{kbetay}):
\begin{equation*}
K_{\beta}(IOw_{B}^{1}(X,Z)) = O(\overline{X},\overline{Z}) - B(w(X),w(Z))
\end{equation*}

and
\begin{equation*}
K_{\beta}(IOw_{B}^{1}(Y,Z)) = O(\overline{Y},\overline{Z}) - w(\widehat{O}(Y,Z)).
\end{equation*}

As $\underline{X} < \underline{Y}$ and $\overline{X} = \overline{Y}$, we have that $w(Y) < w(X)$, and thus,
\begin{equation*}
w(\widehat{O}(Y,Z)) \leq B(w(Y),w(Z)) \leq B(w(X),w(Z)),
\end{equation*}
as $B$ is increasing. So,
\begin{equation*}
K_{\beta}(IOw_{B}^{1}(X,Z)) = O(\overline{X},\overline{Z}) - w(\widehat{O}(X,Z)) \leq O(\overline{Y},\overline{Z}) - B(w(Y),w(Z)) = K_{\beta}(IOw_{B}^{1}(Y,Z)).
\end{equation*}
Then,
\begin{eqnarray*}
\lefteqn{K_{\alpha}(IOw_{B}^{\alpha}(X,Z)) = K_{\alpha}(IOw_{B}^{\alpha}(Y,Z)) \, \, \mbox{and} \, \, K_\beta(IOw_{B}^{\alpha}(X,Z)) \leq K_\beta(IOw_{B}^{\alpha}(Y,Z))}&&\\
&\Rightarrow& IOw_{B}^{\alpha}(X,Z) \leq_{\alpha,\beta} IOw_{B}^{\alpha}(Y,Z).
\end{eqnarray*}

So, we have that $IOw_{B}^{\alpha}(X,Z) \leq_{\alpha,\beta} IOw_{B}^{\alpha}(Y,Z)$, for all $X,Y,Z \in L([0,1])$, such that $\underline{X} < \underline{Y}$ and  $\overline{X} = \overline{Y}$.

\item[d)] $\underline{X} < \underline{Y}$ and  $\overline{X} < \overline{Y}$:

When $\underline{Z} \neq 0$, from Lemma \ref{lema-prkalfa}, it holds that $K_{\alpha}(\widehat{O}(X,Z)) < K_{\alpha}(\widehat{O}(Y,Z))$. So, we have that
\begin{equation*}
K_{\alpha}(IOw_{B}^{\alpha}(X,Z)) < K_{\alpha}(IOw_{B}^{\alpha}(Y,Z)) \Rightarrow IOw_{B}^{\alpha}(X,Z) \leq_{\alpha,\beta} IOw_{B}^{\alpha}(Y,Z).
\end{equation*}
If $\underline{Z} = 0$ and $\overline{Z} \neq 0$, by \textbf{(O2)}, one has that
\begin{equation*}
\widehat{O}(X,Z)= [0, O(\overline{X},\overline{Z})],
\end{equation*}
and
\begin{equation*}
\widehat{O}(Y,Z)= [0, O(\overline{Y},\overline{Z})].
\end{equation*}
Since $\overline{X} < \overline{Y}$ and $O$ is strict, then
\begin{eqnarray*}
\lefteqn{K_{\alpha}(IOw_{B}^{\alpha}(X,Z)) = K_{\alpha}(\widehat{O}(X,Z)) < K_{\alpha}(\widehat{O}(Y,Z)) = K_{\alpha}(IOw_{B}^{\alpha}(Y,Z))}&&\\
&\Rightarrow &IOw_{B}^{\alpha}(X,Z) \leq_{\alpha,\beta} IOw_{B}^{\alpha}(Y,Z).
\end{eqnarray*}

If $\underline{Z} = 0$ and $\overline{Z} = 0$, then
\begin{equation*}
\widehat{O}(X,Z)=IOw_{B}^{\alpha}(X,Z)=[0,0]=IOw_{B}^{\alpha}(Y,Z)=\widehat{O}(X,Z).
\end{equation*}

So, we have that $IOw_{B}^{\alpha}(X,Z) \leq_{\alpha,\beta} IOw_{B}^{\alpha}(Y,Z)$, for all $X,Y,Z \in L([0,1])$, such that $\underline{X} < \underline{Y}$ and  $\overline{X} < \overline{Y}$.

\end{description}

Thus, one can conclude that, for all $X,Y,Z \in L([0,1])$, when $m_{\widehat{O},B}(X,Z)=B(w(X),w(Z))$ and $m_{\widehat{O},B}(Y,Z)= w(\widehat{O}(Y,Z))$, then
\begin{equation*}
X\leq_{Pr}Y \Rightarrow IOw_{B}^{\alpha}(X,Z) \leq_{\alpha,\beta} IOw_{B}^{\alpha}(Y,Z).
\end{equation*}

\end{description}

As verified for all possible scenarios, it holds that  $IOw_{B}^{\alpha}$ is $(\leq_{Pr},\leq_{\alpha,\beta})$-increasing, for all $\alpha,\beta \in [0,1]$ such that $\alpha \neq \beta$.

\vspace{0.5cm}

\textbf{(IOw5)} \begin{eqnarray*}
       w(IOw_{B}^{\alpha}(X,Y)) & = & K_{\alpha}(\widehat{O}(X,Y)) + (1 - \alpha)\cdot m_{\widehat{O},B}(X,Y) - (K_{\alpha}(\widehat{O}(X,Y)) - \alpha \cdot m_{\widehat{O},B}(X,Y))\\
        & = &m_{\widehat{O},B}(X,Y)\\
        & = & \min\{w(\widehat{O}(X,Y)),B(w(X),w(Y))\}\\
        & \leq &  B(w(X),w(Y)).
       \end{eqnarray*}
       Then, it holds that $IOw_{B}^{\alpha}$ is width-limited by $B$ for all $\alpha \in [0,1]$.
\qed
\end{proof}

\section{Proof of Theorem \ref{theo-benja-geral}} \label{ap-proofs2}

\begin{proof}
Consider a commutative, increasing and conjunctive function $B: [0,1]^2 \rightarrow [0,1]$, a strict overlap function $O: [0,1]^2 \rightarrow [0,1]$ and let $\alpha \in (0,1)$, $\beta \in [0,1]$ such that $\alpha \neq \beta$. Observe that, for all $X,Y \in L([0,1])$:
\begin{description}
  \item[(i)] $K_{\alpha}(IOw_{B}^{\alpha}(X,Y)) = O(K_{\alpha}(X), K_{\alpha}(Y))$; 
  \item[(ii)] $w(IOw_{B}^{\alpha}(X,Y)) = \theta = B(B(w(X), w(Y)), B(O(K_{\alpha}(X), K_{\alpha}(Y)),1 - O(K_{\alpha}(X), K_{\alpha}(Y))))$.
\end{description}

So, it is clear that $IOw_{B}^{\alpha}$ is well defined. Now, let us verify if $IOw_{B}^{\alpha}$ respects conditions \textbf{(IOw1)-(IOw5)} from Definition \ref{def-w-iv-ov}.

\vspace{0.5cm}

\textbf{(IOw1)} Immediate, as $O$ and $B$ are commutative;

\vspace{0.5cm}

\textbf{(IOw2)} ($\Rightarrow$) Take $X,Y \in L([0,1])$ and suppose that $IOw_{B}^{\alpha}(X,Y) = [0,0]$. Then, by \textbf{(i)}, we have that \[K_{\alpha}(IOw_{B}^{\alpha}(X,Y)) = K_{\alpha}([0,0]) = 0 = O(K_{\alpha}(X),K_{\alpha}(Y)),\]since $\alpha \in (0,1)$. Thus, by condition \textbf{(O2)}, either $K_{\alpha}(X)=0$ or $K_{\alpha}(Y)=0$, and, therefore, $X \cdot Y = [0,0]$;

($\Leftarrow$) Consider $X,Y \in L([0,1])$ such that $X \cdot Y = [0,0]$. So, $K_{\alpha}(X) \cdot K_{\alpha}(Y) = 0$, since $\alpha \in (0,1)$. Then, by \textbf{(i)} and \textbf{(O2)}, one has that $K_{\alpha}(IOw_{B}^{\alpha}(X,Y)) = O(K_{\alpha}(X),K_{\alpha}(Y)) = 0$, meaning that $IOw_{B}^{\alpha}(X,Y) = [0,0]$;

\vspace{0.5cm}

\textbf{(IOw3)} ($\Rightarrow$)  Take $X,Y \in L([0,1])$ such that $IOw_{B}^{\alpha}(X,Y) = [1,1]$. Then, by \textbf{(i)}, one has that
\begin{equation*}
K_{\alpha}(IOw_{B}^{\alpha}(X,Y)) = K_{\alpha}([1,1]) = 1 = O(K_{\alpha}(X),K_{\alpha}(Y)).
\end{equation*}
By \textbf{(O3)}, $K_{\alpha}(X) \cdot K_{\alpha}(Y) = 1$, since $\alpha \in (0,1)$, meaning that  $X \cdot Y = [1,1]$;

($\Leftarrow$) Consider $X,Y \in L([0,1])$ such that $X \cdot Y = [1,1]$. So, $K_{\alpha}(X) \cdot K_{\alpha}(Y) = 1$, since $\alpha \in (0,1)$. Then, by \textbf{(i)} and \textbf{(O3)}, one has that $K_{\alpha}(IOw_{B}^{\alpha}(X,Y)) = O(K_{\alpha}(X),K_{\alpha}(Y)) = 1$, meaning that $IOw_{B}^{\alpha}(X,Y) = [1,1]$;

\vspace{0.5cm}

\textbf{(IOw4)} Consider $X,Y,Z \in L([0,1])$ such that $X \leq_{\alpha,\beta} Y$ with $\alpha \in (0,1)$, $\beta \in [0,1]$, $\alpha \neq \beta$. By Lemma \ref{lemma-novo}, it is sufficient to consider the cases $\beta=0$ and $\beta =1$. First, for $X <_{\alpha,\beta} Y$ and $\beta = 0$ we have the following possibilities:

\begin{description}
  \item[1)] $X <_{\alpha,0} Y$ and $K_{\alpha}(Z)=0$. Then, $O(K_{\alpha}(X),K_{\alpha}(Z)) = 0 = O(K_{\alpha}(Y),K_{\alpha}(Z))$, and, therefore, since $\alpha \neq 0$, by \textbf{(i)} it holds that $IOw_{B}^{\alpha}(X,Z) =  IOw_{B}^{\alpha}(Y,Z) = [0,0]$;
  \item[2)] $X <_{\alpha,0} Y$ and $K_{\alpha}(Z)>0$. Here, we have the following possibilities:
  \begin{description}
    \item[a)] $K_{\alpha}(X) < K_{\alpha}(Y)$. Since $O$ is strict, by \textbf{(O4)}, one has that $O(K_{\alpha}(X),K_{\alpha}(Z)) < O(K_{\alpha}(Y),K_{\alpha}(Z))$, and, thus, by \textbf{(i)} it follows that $IOw_{B}^{\alpha}(X,Z) <_{\alpha,0} IOw_{B}^{\alpha}(Y,Z)$;
    \item[b)] $K_{\alpha}(X) = K_{\alpha}(Y)$ and $K_{\beta = 0}(X) < K_{\beta = 0}(Y)$. Then, $\underline{X} < \underline{Y} \leq \overline{Y} < \overline{X}$, meaning that $w(X) > w(Y)$. So, by \textbf{(i)},
        \begin{equation*}
        K_{\alpha}(IOw_{B}^{\alpha}(X,Z)) = O(K_{\alpha}(X),K_{\alpha}(Z)) = O(K_{\alpha}(Y),K_{\alpha}(Z)) = K_{\alpha}(IOw_{B}^{\alpha}(Y,Z)),
          \end{equation*}
        and
        \begin{eqnarray*}
        \lefteqn{K_{\beta = 0}(IOw_{B}^{\alpha}(X,Z)) = K_{\alpha}(IOw_{B}^{\alpha}(X,Z)) - \alpha \cdot w(IOw_{B}^{\alpha}(X,Z)) \, \, \mbox{by Definition \ref{def-alfabeta}}}&&\\
        &=& K_{\alpha}(IOw_{B}^{\alpha}(X,Z)) - \alpha \cdot B(B(w(X), w(Z)), B(O(K_{\alpha}(X), K_{\alpha}(Z)),1 - O(K_{\alpha}(X), K_{\alpha}(Z))))\\
        &&  \hspace{12cm}\mbox{by \textbf{(ii)}}\\
        &\leq&  K_{\alpha}(IOw_{B}^{\alpha}(Y,Z)) - \alpha \cdot B(B(w(Y), w(Z)), B(O(K_{\alpha}(Y), K_{\alpha}(Z)),1 - O(K_{\alpha}(Y), K_{\alpha}(Z))))\\
        &=& K_{\alpha}(IOw_{B}^{\alpha}(Y,Z)) - \alpha \cdot  w(IOw_{B}^{\alpha}(Y,Z))\\
        & =& K_{\beta = 0}(IOw_{B}^{\alpha}(Y,Z)),
          \end{eqnarray*}
         as $B$ is increasing. Therefore, $IOw_{B}^{\alpha}(X,Z) \leq_{\alpha,0} IOw_{B}^{\alpha}(Y,Z)$.
  \end{description}
\end{description}

When $X = Y$, it is immediate that $IOw_{B}^{\alpha}(X,Z) = IOw_{B}^{\alpha}(Y,Z)$. Then, for $\beta = 0$ it holds that
\begin{equation*}
IOw_{B}^{\alpha}(X,Z) \leq_{\alpha,0} IOw_{B}^{\alpha}(Y,Z).
\end{equation*}
The proof for $\beta = 1$ can be obtained analogously.

\vspace{0.5cm}

\textbf{(IOw5)} By \textbf{(ii)}, since $B$ is conjunctive, it holds that\begin{eqnarray*}
      \lefteqn{w(IOw_{B}^{\alpha}(X,Y)) = \theta}&&\\
     &=& B(B(w(Y), w(Z)), B(O(K_{\alpha}(Y), K_{\alpha}(Z)),1 - O(K_{\alpha}(Y) \leq  B(w(X),w(Y)),
      \end{eqnarray*}
Then, it holds that $IOw_{B}^{\alpha}$ is width-limited by $B$ for all $\alpha \in (0,1)$.
\qed
\end{proof}

\section{Proof of Theorem \ref{theo-conzdenko}} \label{ap-proofs3}

\begin{proof}
Consider a commutative aggregation function $B: [0,1]^2 \rightarrow [0,1]$, a strict overlap function $O: [0,1]^2 \rightarrow [0,1]$ and let $\alpha \in (0,1)$ and $\beta \in [0,1]$ such that $\alpha \neq \beta$. Observe that it is immediate that $IOw_{B}^{\alpha}$ is well defined. In fact, considering that $IOw_{B}^{\alpha}(X,Y)= R$, one has that
$w(R) = m_{IF_{O,B}^{\alpha},B}(X,Y)$ which, by Definition \ref{moa}, is uniquely defined for the pair $(IF_{O,B}^{\alpha}, B)$. As $K_{\alpha}(R)=O(K_{\alpha}(X),K_{\alpha}(Y))$, then, it follows that $\underline{R} = K_{\alpha}(R) - \alpha \cdot w(R)$ and $\overline{R} = K_{\alpha}(R) + (1 - \alpha) \cdot w(R)$.

Now, let us verify if $IOw_{B}^{\alpha}$ respects conditions \textbf{(IOw1)-(IOw5)} from Definition \ref{def-w-iv-ov}.

\vspace{0.5cm}

\textbf{(IOw1)} Observe that, since $O$ and $B$ are commutative, then $IF_{O,B}^{\alpha}$ is commutative, as well as $m_{IF_{O,B}^{\alpha},B}$. Then, it is immediate that $IOw_{B}^{\alpha}$ is commutative;

\vspace{0.5cm}

\textbf{(IOw2)} ($\Rightarrow$) Take $X,Y \in L([0,1])$ and suppose that $IOw_{B}^{\alpha}(X,Y) = R = [0,0]$. Then, by \textbf{(i)}, we have that
\begin{equation*}
K_{\alpha}(R) = K_{\alpha}([0,0]) = 0 = O(K_{\alpha}(X),K_{\alpha}(Y)),
\end{equation*}
since $\alpha \in (0,1)$. Thus, by condition \textbf{(O2)}, either $K_{\alpha}(X)=0$ or $K_{\alpha}(Y)=0$, and, therefore, $X \cdot Y = [0,0]$;

($\Leftarrow$) Consider $X,Y \in L([0,1])$ such that $X \cdot Y = [0,0]$. So, $K_{\alpha}(X) \cdot K_{\alpha}(Y) = 0$, since $\alpha \in (0,1)$. Then, by \textbf{(i)} and \textbf{(O2)}, one has that
\begin{equation*}
K_{\alpha}(R) = O(K_{\alpha}(X),K_{\alpha}(Y)) = 0,
\end{equation*}
meaning that $IOw_{B}^{\alpha}(X,Y) = R = [0,0]$;

\vspace{0.5cm}

\textbf{(IOw3)} ($\Rightarrow$)  Take $X,Y \in L([0,1])$ such that $IOw_{B}^{\alpha}(X,Y) = R = [1,1]$. Then, by \textbf{(i)}, one has that
\begin{equation*}
K_{\alpha}(R) = K_{\alpha}([1,1]) = 1 = O(K_{\alpha}(X),K_{\alpha}(Y)).
\end{equation*}
By \textbf{(O3)}, $K_{\alpha}(X) \cdot K_{\alpha}(Y) = 1$, since $\alpha \in (0,1)$, meaning that  $X \cdot Y = [1,1]$;

($\Leftarrow$) Consider $X,Y \in L([0,1])$ such that $X \cdot Y = [1,1]$. So, $K_{\alpha}(X) \cdot K_{\alpha}(Y) = 1$, since $\alpha \in (0,1)$. Then, by \textbf{(i)} and \textbf{(O3)}, one has that
\begin{equation*}
K_{\alpha}(R) = O(K_{\alpha}(X),K_{\alpha}(Y)) = 1,
\end{equation*}
meaning that $IOw_{B}^{\alpha}(X,Y) = R = [1,1]$;

\vspace{0.5cm}


\textbf{(IOw4)} Consider $X,Y,Z \in L([0,1])$ such that $X \leq_{\alpha,\beta} Y$ with $\alpha \in (0,1)$, $\beta \in [0,1]$, such that $\alpha \neq \beta$. By Lemma \ref{lemma-novo}, it is sufficient to consider the cases $\beta=0$ and $\beta =1$. First, for $X <_{\alpha,\beta} Y$ and $\beta = 0$ we have the following possibilities:

\begin{description}
  \item[1)] $X <_{\alpha,0} Y$ and $K_{\alpha}(Z)=0$. Then, $O(K_{\alpha}(X),K_{\alpha}(Z)) = 0 = O(K_{\alpha}(Y),K_{\alpha}(Z))$, and, therefore, since $\alpha \neq 0$, by \textbf{(i)} it holds that $IOw_{B}^{\alpha}(X,Z) = [0,0] = IOw_{B}^{\alpha}(Y,Z)$;
  \item[2)] $X <_{\alpha,0} Y$ and $K_{\alpha}(Z)>0$. Here, we have the following possibilities:
  \begin{description}
    \item[a)] $K_{\alpha}(X) < K_{\alpha}(Y)$. Since $O$ is strict, by \textbf{(O4)}, one has that $O(K_{\alpha}(X),K_{\alpha}(Z)) < O(K_{\alpha}(Y),K_{\alpha}(Z))$, and, thus, by \textbf{(i)} it follows that $IOw_{B}^{\alpha}(X,Z) <_{\alpha,0} IOw_{B}^{\alpha}(Y,Z)$;
    \item[b)] $K_{\alpha}(X) = K_{\alpha}(Y)$ and $K_{\beta = 0}(X) < K_{\beta = 0}(Y)$. Then, $\underline{X} < \underline{Y} \leq \overline{Y} < \overline{X}$, meaning that $w(X) > w(Y)$ and, therefore, by Definition \ref{def-lamda}, $\lambda_{\alpha}(X) > \lambda_{\alpha}(Y)$. So, by \textbf{(i)},
        \begin{equation*}
        K_{\alpha}(IOw_{B}^{\alpha}(X,Z)) = O(K_{\alpha}(X),K_{\alpha}(Z)) = O(K_{\alpha}(Y),K_{\alpha}(Z)) = K_{\alpha}(IOw_{B}^{\alpha}(Y,Z)),
          \end{equation*}
        and
        \begin{eqnarray*}
        \lefteqn{K_{\beta = 0}(IOw_{B}^{\alpha}(X,Z)) = K_{\alpha}(IOw_{B}^{\alpha}(X,Z)) - \alpha \cdot w(IOw_{B}^{\alpha}(X,Z)) \, \, \mbox{by Definition \ref{def-alfabeta}}}&&\\
        &=& K_{\alpha}(IOw_{B}^{\alpha}(X,Z)) - \alpha \cdot m_{IF_{O,B}^{\alpha},B}(X,Z) \, \, \mbox{by \textbf{(ii)}}\\
        &=& K_{\alpha}(IOw_{B}^{\alpha}(X,Z)) - \alpha \cdot \min\{B(w(X),w(Z)),B(\lambda_{\alpha}(X),\lambda_{\alpha}(Z))\cdot d_{\alpha}(K_{\alpha}(IOw_{B}^{\alpha}(X,Z)))\} \\
        && \hspace{10cm} \mbox{by Definition \ref{moa}}\\
        &\leq& K_{\alpha}(IOw_{B}^{\alpha}(Y,Z)) - \alpha \cdot \min\{B(w(Y),w(Z)),B(\lambda_{\alpha}(Y),\lambda_{\alpha}(Z))\cdot d_{\alpha}(K_{\alpha}(IOw_{B}^{\alpha}(Y,Z)))\}\\
        &=& K_{\alpha}(IOw_{B}^{\alpha}(Y,Z)) - \alpha \cdot m_{IF_{O,B}^{\alpha},B}(Y,Z) \, \, \mbox{by Definition \ref{moa}}\\
        &=& K_{\beta = 0}(IOw_{B}^{\alpha}(Y,Z)),
          \end{eqnarray*}
         as $B$ is increasing. Therefore, $IOw_{B}^{\alpha}(X,Z) \leq_{\alpha,0} IOw_{B}^{\alpha}(Y,Z)$.
  \end{description}
\end{description}

When $X = Y$, it is immediate that $IOw_{B}^{\alpha}(X,Z) = IOw_{B}^{\alpha}(Y,Z)$. Then, for $\beta = 0$ it holds that
\begin{equation*}
IOw_{B}^{\alpha}(X,Z) \leq_{\alpha,0} IOw_{B}^{\alpha}(Y,Z).
\end{equation*}

The proof for $\beta = 1$ can be obtained analogously.

\vspace{0.5cm}

\textbf{(IOw5)} By \textbf{(ii)} and Definition \ref{moa}, it holds that\begin{eqnarray*}
       \lefteqn{w(IOw_{B}^{\alpha}(X,Y)) = m_{IF_{O,B}^{\alpha},B}(X,Y)}&&\\
       &=& \min\{B(w(X),w(Y)),B(\lambda_{\alpha}(X),\lambda_{\alpha}(Y))\cdot d_{\alpha}(K_{\alpha}(R))\}\\
       &\leq& B(w(X),w(Y)).
       \end{eqnarray*}
Then, it holds that $IOw_{B}^{\alpha}$ is width-limited by $B$ for all $\alpha \in (0,1)$.
\qed
\end{proof}

\bibliographystyle{elsart-num-sort}

\end{document}